%% file: main.tex
\documentclass[11pt]{article}
\pdfoutput=1
\usepackage{nickstyle}
\usepackage[numbers]{natbib}
\usepackage{ifpdf}
\usepackage{xcolor}
\usepackage{enumitem}
\usepackage[makeroom]{cancel}
\usepackage{algorithm,algorithmicx,algpseudocode}
\usepackage{calc}
\usepackage{hyperref}
\usepackage{framed}
\usepackage[textsize=scriptsize]{todonotes}
\usepackage{mathtools}
\usepackage{times}
\usepackage{amsthm}
\usepackage{thmtools}
\usepackage{thm-restate}

\newif\iffull
\fulltrue

\setmarginsrb{1in}{1in}{1in}{50pt}{0pt}{0pt}{12pt}{22pt}

\algnewcommand{\MyState}[1]{\State
\parbox[t]{\dimexpr\linewidth-\ALG@thistlm}{\hangindent=0pt\strut\hangafter=1#1\strut}}
\makeatother
\algdef{SxN}{Myprocedure}{EndMyprocedure}[2]{\algorithmicprocedure\ \textproc{#1}\ifthenelse{\equal{#2}{}}{}{(#2)}}
\algdef{SxN}{MyFor}{MyEndFor}[1]{\algorithmicfor\ #1\ \algorithmicdo}

\newlength\abovesectionskip
\newlength\belowsectionskip
\def\sectionfont{\normalfont\Large\bfseries}
\newlength\abovesubsectionskip
\newlength\belowsubsectionskip
\def\subsectionfont{\normalfont\large\bfseries}
\newlength\abovesubsubsectionskip
\newlength\belowsubsubsectionskip
\def\subsubsectionfont{\normalfont\normalsize\bfseries}
\newlength\aboveparagraphskip
\newlength\belowparagraphskip
\def\paragraphfont{\normalfont\normalsize\bfseries}

\makeatletter
\def\section{\@startsection{section}{1}{\z@}{-\abovesectionskip}%
               {\belowsectionskip}{\sectionfont}}
\def\subsection{\@startsection{subsection}{2}{\z@}{-\abovesubsectionskip}%
                  {\belowsubsectionskip}{\subsectionfont}}
\def\subsubsection{\@startsection{subsubsection}{3}{\z@}%
                     {-\abovesubsubsectionskip}{\belowsubsubsectionskip}%
                     {\subsubsectionfont}}
\def\paragraph{\@startsection{paragraph}{4}{\z@}{-\aboveparagraphskip}%
                 {-\belowparagraphskip}{\paragraphfont}}
\makeatother

 \makeatletter
 \renewenvironment{align*}{%
   \abovedisplayskip 5pt plus 1pt%
   \belowdisplayskip 5pt plus 1pt%
   \start@align\@ne\st@rredtrue\m@ne
 }{%
   \endalign
 }
\let\stdequation\equation
\renewcommand*\equation{%
  \abovedisplayskip 5pt plus 1pt%
  \belowdisplayskip 5pt plus 1pt%
  \stdequation}
\DeclareRobustCommand{\[}{
  \abovedisplayskip 5pt plus 1pt%
  \belowdisplayskip 5pt plus 1pt%
  \begin{equation*}
}
\makeatother

\renewenvironment{itemize}{
    \begin{list}{$\bullet$}{
        \setlength{\labelsep}{6pt}\setlength{\itemindent}{0mm}\setlength{\labelwidth}{3mm}
        \setlength{\leftmargin}{30pt}
        \setlength{\itemsep}{2pt}\setlength{\parsep}{0mm}
        \setlength{\topsep}{4pt}\setlength{\listparindent}{0pt}
        \setlength{\parskip}{0pt}
    }
}
{
    \end{list}
}

\newcommand\mymathbox[1]{
  \begingroup
  \fcolorbox{yellow}{yellow}{$\displaystyle#1$}
  \endgroup
}

\newcommand{\me}{\mathrm{e}}

\newcommand{\Grad}{\nabla}

\newcommand{\Remark}[1]{Remark~\ref{rem:#1}}
\newcommand{\RemarkName}[1]{\label{rem:#1}}

\newcommand{\repeatclaim}[2]{\medskip\noindent\textbf{#1. }{#2} \medskip}

\newcommand{\repeatclaimwithname}[3]{\medskip\noindent\textbf{#1}{ (#2).\hspace{.25em}{#3}}\medskip}

\renewcommand{\th}{\ifmmode{^{\textrm{th}}}\else{\textsuperscript{th}\ }\fi}

\makeatletter
\newcommand{\myfootnote}[1]{\ensuremath{^{\@fnsymbol{#1}}}}
\renewcommand\maketitle{
	\mbox{}
    
    \begingroup\sffamily
    \begin{center}\textbf{\Large{\@title}}\end{center}
    \begin{center}
    \begin{tabular}{cc}
      Nicholas J.~A.~Harvey\myfootnote{1} & Christopher Liaw\myfootnote{1} \\
      Yaniv Plan\myfootnote{2} & Sikander Randhawa\myfootnote{1}
      \end{tabular} \\[6pt]
      \mbox{}\myfootnote{1} Department of Computer Science \\
      \mbox{}\myfootnote{2} Department of Mathematics \\
      University of British Columbia
    \end{center}
  \vspace{0.3cm}
  \endgroup
}
\makeatother

\date{}

\title{Tight analyses for non-smooth stochastic gradient descent}

\begin{document}

\pagestyle{empty}
\maketitle

\begin{abstract}
Consider the problem of minimizing functions that are Lipschitz and strongly
convex, but not necessarily differentiable.
We prove that after $T$ steps of stochastic gradient descent,
the error of the final iterate is $O(\log(T)/T)$ \emph{with high probability}.
We also construct a function from this class for which the error of the final iterate of \emph{deterministic} gradient descent is $\Omega(\log(T)/T)$.
This shows that the upper bound is tight and that, in this setting,
the last iterate of stochastic gradient descent has the same general error rate (with high probability) as deterministic gradient descent.
This resolves both open questions posed by \citet{ShamirQuestion}.

An intermediate step of our analysis proves that the suffix averaging method
achieves error $O(1/T)$ \emph{with high probability}, which is optimal (for any first-order optimization method).
This improves results of \citet{Rakhlin} and \citet{HK14}, both of which
achieved error $O(1/T)$, but only in expectation,
and achieved a high probability error bound of $O(\log \log(T)/T)$,
which is suboptimal.

\iffull
We prove analogous results for functions that are Lipschitz and convex, but not necessarily strongly convex or differentiable. After $T$ steps of stochastic gradient descent,
the error of the final iterate is $O(\log(T)/\sqrt{T})$ \emph{with high probability}, and there exists a function for which the error of the final iterate of \emph{deterministic} gradient descent is $\Omega(\log(T)/\sqrt{T})$.
\fi
\end{abstract}

\newpage \pagestyle{plain}\setcounter{page}{1}

\input{intro}
\input{prelim}
\input{results}
\input{techniques}
\input{lbshort}
\input{ubshortnew}

\clearpage 
\appendix
\input{appendix_standard}

\input{appendix_lb_omitted}
\input{appendix_fanv2}
\input{appendix_recursive_mgf}

\input{appendix_ub_omitted}
\input{generalizations}
\input{appendix_lb_delta}

\clearpage
\bibliographystyle{plainnat}
\bibliography{GD}

\end{document}

%% file: intro.tex
\section{Introduction}

Stochastic gradient descent (SGD) is one of the oldest randomized algorithms, dating back to 1951~\cite{RM51}.
It is a very simple and widely used iterative method for minimizing a function.
In a nutshell, the method works by querying an oracle for a noisy estimate of a subgradient,
then taking a small step in the opposite direction.
The simplicity and effectiveness of this algorithm has established it both as an essential tool for applied machine learning \cite{sag,svrg},
and as a versatile framework for theoretical algorithm design.

In theoretical algorithms, SGD often appears in the guise of \emph{coordinate descent}, an important special case in which each gradient estimate has a single non-zero coordinate.
Some of the fast algorithms for Laplacian linear systems \cite{LS13,KOSZ13} are based on coordinate descent (and the related Kaczmarz method \cite{SV06}).
Multi-armed Bandits were discovered years ago to be a perfect setting for coordinate descent~\cite{ACFS95}: 
the famous Exp3 algorithm combines coordinate descent and the multiplicative weight method.
Recent work on the geometric median problem \cite{CLMPS16} gave a sublinear time algorithm based on SGD, and very recently a new privacy amplification technique \cite{FMTT18} has been developed that injects noise to the subgradients while executing SGD.
Surveys and monographs discussing gradient descent
and aimed at a theoretical CS audience
include Bansal and Gupta \cite{BansalGupta}, \citet{Bubeck}, \citet{Hazan},
and \citet{Vishnoi18}.

The efficiency of SGD is usually measured by the rate of decrease of the \emph{error} ---
the difference in value between the algorithm's output and the true minimum.
The optimal error rate is known under various assumptions on $f$, the function to be minimized.
In addition to convexity, common assumptions are that $f$ is \emph{smooth} (gradient is Lipschitz)
or \emph{strongly convex} (locally lower-bounded by a quadratic).
Strongly convex functions often arise due to regularization,
whereas smooth functions can sometimes be obtained by smoothening approximations (e.g., convolution).
Existing analyses~\cite{NJLS09} show that, after $T$ steps of SGD,
the expected error of the final iterate is $O(1/\sqrt{T})$ for \emph{smooth} functions,
and $O(1/T)$ for functions that are both \emph{smooth} and strongly convex;
furthermore, both of these error rates are optimal without further assumptions.

The \emph{non-smooth} setting is the focus of this paper.
In theoretical algorithms and discrete optimization, the convex functions that arise are often non-smooth.
For example, the objective for the geometric median problem is a (sum of) 2-norms \cite{CLMPS16}, so Lipschitz but not smooth.
Similarly, formulating the minimum $s$-$t$ cut problem as convex minimization \cite{LRS13}, the objective is a 1-norm, so Lipschitz but not smooth.
In machine learning, the objective for regularized support vector machines \cite{Pegasos} is strongly convex but not smooth.

A trouble with the non-smooth setting is that the error of (even deterministic) gradient descent need not decrease monotonically with $T$,
so it is not obvious how to analyze the error of the final iterate.
A workaround, known as early as \cite{NY83}, is to output the \emph{average} of the iterates.
Existing analyses of SGD show that the \emph{expected} error of the average is
$\Theta(1/\sqrt{T})$ for Lipschitz functions~\cite{NY83}, which is optimal,
whereas for functions that are also strongly convex~\cite{HAK07,Rakhlin} the average has error $\Theta(\log(T)/ T)$ with high probability,
which is not the optimal rate.
An alternative algorithm, more complicated than SGD, was discovered by~\citet{HK14};
it achieves the optimal \emph{expected} error rate of $O(1/T)$.
\emph{Suffix averaging}, a simpler approach in which the last \emph{half} of the SGD iterates are averaged, was also shown to achieve \emph{expected} error $O(1/T)$ \cite{Rakhlin},
although implementations can be tricky or memory intensive if the number of iterations $T$ is unknown a priori.
Non-uniform averaging schemes with optimal expected error rate and simple implementations are also known \cite{LSB12,ShamirZhang}, although the solutions may be less interpretable.

\citet{ShamirQuestion} asked the very natural question of whether the \emph{final} iterate of SGD achieves the optimal rate in the non-smooth scenario, as it does in the smooth scenario.
If true, this would yield a very simple, implementable and interpretable form of SGD.
Substantial progress on this question was made by \citet{ShamirZhang},
who showed that the final iterate has
\emph{expected} error $O(\log(T)/ \sqrt{T})$ for Lipschitz $f$,
and $O(\log(T)/ T)$ for strongly convex $f$.
Both of these bounds are a $\log(T)$ factor worse than the optimal rate, so \citet{ShamirZhang} write

\begin{quote}
An important open question is whether the
$O(\log(T)/T)$ [\emph{expected}] rate we obtained on [the last iterate],
for strongly-convex problems, is tight.
This question is important, because running SGD for
$T$ iterations, and returning the last iterate, is a
very common heuristic.
In fact, even for the simpler case of (non-stochastic)
gradient descent, we do not know whether the behavior
of the last iterate... is tight.
\end{quote}
Our work shows that the $\log(T)$ factor is necessary,
both for Lipschitz functions and for strongly convex functions, even for \emph{non}-stochastic gradient descent. 
So both of the expected upper bounds due to Shamir and Zhang are actually tight.
This resolves the first question of \citet{ShamirQuestion}.
In fact, we show a much stronger statement: \emph{any convex combination}
of the last $k$ iterates must incur a $\log(T/k)$ factor.
Thus, suffix averaging must average a constant fraction of the iterates to achieve the optimal rate.

High probability bounds on SGD are somewhat scarce; most of the literature proves bounds in expectation, which is of course easier.
A common misconception is that picking the best of several independent trials of SGD would yield high-probability bounds, but this approach is not as efficient as it might seem\footnote{
	It is usually the case that selecting the best of many independent trials is very inefficient.
	Such a scenario, which is very common in uses of SGD, arises if $f$ is defined as $\sum_{i=1}^m f_i$ or $\expectover{\omega}{f_\omega}$.
    In such scenarios, evaluating $f$ exactly could be inefficient,
    and even estimating it to within error $1/T$ requires $\Theta(T^2)$ samples via a Hoeffding bound, whereas SGD uses only $O(T)$ samples.
    }.
So it is both interesting and useful that high-probability bounds hold for a single execution of SGD.
Some known high-probability bounds for the strongly convex setting include \cite{KT08}, for uniform averaging, and \cite{HK14,Rakhlin}, which give a suboptimal bound of $O(\log\log(T)/T)$ for suffix averaging (and a variant thereof).
In this work, we give two high probability bounds on the error of
\iffull
	SGD for strongly convex functions:
\fi
$O(1/T)$ for suffix averaging and $O(\log(T)/T)$ for the final iterate. Both of these are tight.
(Interestingly, the former is used as an ingredient for the latter.)
The former answers a question of \citet[\S 6]{Rakhlin}, and the latter resolves the second question of \citet{ShamirQuestion}.
\iffull
	For Lipschitz functions, we prove a high probability bound of $O(\log(T)/\sqrt{T})$ for the final iterate, which is also tight.
\fi

Our work can also be seen as extending a line of work on understanding the difference between an \emph{average} of the iterates or the last iterate of an iterative process.
For instance, one of the most important results in game theory is that the multiplicative weights update algorithm converges to an equilibrium \cite{FS99}, i.e.~the set of players are required to play some sort of ``coordinated average'' of their past strategies.
Recently, \cite{BP18} studied the convergence behaviour of players' individual strategies and found that the strategies \emph{diverge} and hence, coordination (i.e.~averaging) is needed to obtain an equilibrium.
In a similar spirit, our work shows that the iterates of gradient descent have a sub-optimal convergence rate, at least for non-smooth convex functions, and thus, some form of averaging is needed to achieve the optimal rate.
It is an interesting direction to see whether or not this is necessary in other iterative methods as well.
For instance, the multiplicative weights update algorithm can be used to give an iterative algorithm for maximum flow \cite{CKMST11}, or linear programming in general \cite{AroraHazanKale,PST95}, but also requires some form of averaging.
We hope that this paper contributes to a better understanding on when averaging is necessary in iterative processes.

%% file: prelim.tex
\section{Preliminaries}
Let $\cX$ be a closed, convex subset of $\bR^n$, $f \colon \cX \to \bR$ be a convex function, and $\partial f(x)$ the subdifferential of $f$ at $x$.
Our goal is to solve the convex program $\min_{x \in \cX} f(x)$.
We assume that $f$ is not explicitly represented.
Instead, the algorithm is allowed to query $f$ via a stochastic gradient oracle, i.e.,~if the oracle is queried at $x$ then it returns $\hat{g} = g - \hat{z}$ where $g \in \partial f(x)$ and $\expect{\hat{z}} = 0$ conditioned on all past calls to the oracle.
The set $\cX$ is represented by a projection oracle, which returns the point in $\cX$ closest in Euclidean norm to a given point $x$. 
We say that $f$ is $\alpha$-strongly convex if
\begin{equation}
\EquationName{strongly_convex_def}
f(y) ~\geq~ f(x) + \inner{g}{y-x} + \frac{\alpha}{2}\norm{y - x}^2 \quad \forall y, x \in \cX, g \in \partial f(x).
\end{equation}
Throughout this paper, $\norm{\cdot}$ denotes the \emph{Euclidean} norm in $\bR^n$
and $[T]$ denotes the set $\set{1,\ldots,T}$.

We say that $f$ is $L$-Lipschitz if $\norm{g} \leq L$ for all $x \in \cX$ and $g \in \partial f(x)$.
For the remainder of this paper, unless otherwise stated, we make the assumption that $\alpha = 1$ and $L = 1$; this is only a normalization assumption and is without loss of generality (see \Appendix{Generalizations}).
For the sake of simplicity, we also assume that $\norm{\hat{z}} \leq 1$ a.s.~although our arguments generalize to the setting when $\hat{z}$ are sub-Gaussian (see \Appendix{Generalizations}).

Let $\Pi_\cX$ denote the projection operator on $\cX$.
The (projected) stochastic gradient algorithm is given in \Algorithm{SGD}.
Notice that there the algorithm maintains a sequence of points and there are several strategies to output a single point.
The simplest strategy is to simply output $x_{T+1}$.
However, one can also consider averaging all the iterates \cite{PJ92,Ruppert88} or averaging only a fraction of the final iterates \cite{Rakhlin}.
Notice that the algorithm also requires the user to specify a sequence of step sizes.
The optimal choice of step size is known to be $\eta_t = \Theta(1/t)$ for strongly convex functions 
\iffull
	\cite{NJLS09,Rakhlin},
	and $\eta_t = \Theta(1/\sqrt{t})$ for Lipschitz  functions.
\fi
For our analyses, we will use a step size of
\iffull
	$\eta_t = 1/t$ for strongly convex functions and $\eta_t = 1/\sqrt{t}$ for Lipschitz functions.
\fi

\begin{algorithm}
\caption{Projected stochastic gradient descent for minimizing a non-smooth, convex function.}
\AlgorithmName{SGD}
\begin{algorithmic}[1]
\Myprocedure{StochasticGradientDescent}{$\cX \subseteq \bR^n$,\, $x_1 \in \cX$, step sizes $\eta_1, \eta_2, \ldots$}
\MyFor{$t \leftarrow 1,\ldots,T$}
	\State Query stochastic gradient oracle at $x_t$ for $\hat{g}_t$ such that $\expectg{\hat{g_t}}{\hat{g}_{1}, \ldots, \hat{g}_{t-1}} \in \partial f(x_t)$
    \State $y_{t+1} \leftarrow x_t - \eta_t \hat{g}_t$
    ~(take a step in the opposite direction of the subgradient)
    \State $x_{t+1} \leftarrow \Pi_\cX(y_{t+1})$
    ~(project $y_{t+1}$ onto the set $\cX$)
\MyEndFor
\State \textbf{return}
either $
\begin{cases}
x_{T+1} &\text{(final iterate)} \\
\frac{1}{T+1} \sum_{t=1}^{T+1} x_t &\text{(uniform averaging)} \\
\frac{1}{T/2 + 1} \sum_{t=T/2+1}^{T+1} x_t &\text{(suffix averaging)}
\end{cases}
$
\EndMyprocedure
\end{algorithmic}
\end{algorithm}

%% file: results.tex
\section{Our Contributions}

Our main results are bounds on the error of the final iterate of
\iffull
	stochastic gradient descent for non-smooth, convex functions.
    
  \paragraph{Strongly convex and Lipschitz functions.}
  We prove an $\Omega(\log(T)/T)$ lower bound, even in the non-stochastic case, and an $O(\log(T) \log(1/\delta)/T)$ upper bound with probability $1-\delta$.

  \paragraph{Lipschitz functions.}
  We prove an $\Omega(\log(T)/\sqrt{T})$ lower bound, even in the non-stochastic case, and an $O(\log(T) \log(1/\delta)/\sqrt{T})$ upper bound with probability $1-\delta$.
\fi

\iffull
\subsection{High probability upper bounds}
\fi

\begin{theorem}
\TheoremName{FinalIterateHighProbability}
Suppose $f$ is $1$-strongly convex and $1$-Lipschitz.
Suppose that $\hat{z}_t$
(i.e., $\expect{\hat{g_t}}-\hat{g_t}$, the noise of the stochastic gradient oracle)
has norm at most $1$ almost surely.
Consider running \Algorithm{SGD} for $T$ iterations with step size $\eta_t = 1/t$.
Let $x^* = \argmin_{x \in \cX} f(x)$.
Then, with probability at least $1-\delta$,
\[
f(x_{T+1}) - f(x^*) ~\leq~ O \bigg ( \frac{\log(T) \log(1/\delta)}{T}\bigg).
\]
\end{theorem}

\iffull
\begin{theorem}
\TheoremName{FinalIterateHighProbabilityLipschitz}
Suppose $f$ is and $1$-Lipschitz and $\cX$ has diameter $1$.
Suppose that $\hat{z}_t$
(i.e., $\expect{\hat{g_t}}-\hat{g_t}$, the noise of the stochastic gradient oracle)
has norm at most $1$ almost surely.
Consider running \Algorithm{SGD} for $T$ iterations with step size $\eta_t = 1/\sqrt{t}$.
Let $x^* = \argmin_{x \in \cX} f(x)$.
Then, with probability at least $1-\delta$,
\[
f(x_{T+1}) - f(x^*) ~\leq~ O \bigg ( \frac{\log(T) \log(1/\delta)} {\sqrt{T}}\bigg).
\]
\end{theorem}
\fi

The assumptions on the strong convexity parameter,  Lipschitz parameter, and diameter are without loss of generality; see \Appendix{Generalizations}.
The bounded noise assumption for the stochastic gradient oracle is made only for simplicity; our analysis can be made to go through if one relaxes the a.s.~bounded condition to a sub-Gaussian condition.
We also remark that a linear dependence on $\log(1/\delta)$ is necessary for strongly convex functions; see \Appendix{lb_delta}.

Our main probabilistic tool to prove \Theorem{FinalIterateHighProbability}
\iffull
	and \Theorem{FinalIterateHighProbabilityLipschitz}
\fi
is a new extension of the classic Freedman inequality~\cite{F75} to a setting in which the martingale exhibits a curious phenomenon. 
Ordinarily a martingale is roughly bounded by the square root of its
total conditional variance (this is the content of Freedman's inequality).
We consider a setting in which the total conditional variance\footnote{
As stated, \Theorem{FanV2} assumes a conditional
sub-Gaussian bound on the martingale difference sequence,
whereas Freedman assumes both a conditional variance bound and an almost-sure bound.
These assumptions are easily interchangeable
in both our proof and Freedman's proof.
For example, Freedman's inequality with the sub-Gaussian assumption
appears in \cite[Theorem~2.6]{Fan15}.
}
is itself bounded by (a linear transformation of) the martingale.
We refer to this as a ``chicken and egg'' phenomenon.

\newcommand{\fantwo}{
Let $\{ d_i, \cF_i \}_{i=1}^n$ be a martingale difference sequence.
Suppose $v_{i-1}$, $i \in [n]$ are positive and $\cF_{i-1}$-measurable random variables such that $\expectg{\exp(\lambda d_i)}{\cF_{i-1}} \leq \exp\left(\frac{\lambda^2}{2} v_{i-1}\right)$ for all $i \in [n],\, \lambda > 0$.
Let $S_t = \sum_{i=1}^t d_i$ and $V_t = \sum_{i=1}^t v_{i-1}$.
Let $\alpha_i \geq 0$ and set $\alpha = \max_{i\in[n]} \alpha_i$. Then
\[
\prob{\: \bigcup_{t=1}^n \left\{S_t \geq x \text{ and } V_t \leq
\mymathbox{\sum_{i=1}^t \alpha_i d_i} + \beta \right\} \:} ~\leq~ \exp\left( -\frac{x}{\mymathbox{4\alpha} + 8\beta / x} \right)
\qquad\forall x, \beta > 0.
\]
}

\begin{theorem}[Generalized Freedman]
\TheoremName{FanV2}
\fantwo
\end{theorem}

The proof of \Theorem{FanV2} appears in \Appendix{FanV2}.
Freedman's Inequality \cite{F75} (as formulated in \cite[Theorem~2.6]{Fan15}, up to constants) simply omits the terms highlighted in yellow, i.e., it sets $\alpha = 0$.

\renewcommand{\mymathbox}[1]{#1}

\iffull
	\subsection{Lower bounds}

\fi

\begin{theorem}
\TheoremName{FinalIterateLowerBound}
For any $T$, there exists a convex function $f_T : \cX \rightarrow \bR$,
where $\cX$ is the unit Euclidean ball in $\bR^T$,
such that $f_T$ is $3$-Lipschitz and $1$-strongly convex, and satisfies the following.
Suppose that \Algorithm{SGD} is executed
from the initial point $x_1 = 0$ with step sizes $\eta_t = 1/t$.
Let $x^* = \argmin_{x \in \cX} f_T(x)$.
Then
\begin{equation}
\EquationName{LastIterateLargeSC}
f_T(x_T) - f_T(x^*) ~\geq~ \frac{\log T}{4 T}.
\end{equation}
More generally, any weighted average $\bar{x}$ of the last $k$ iterates has
\begin{equation}
\EquationName{EveryIterateLargeSC}
f_T(\bar{x}) - f_T(x^*) ~\geq~ \frac{\ln(T) - \ln(k)}{4 T}.
\end{equation}
Thus, suffix averaging must average a constant fraction of iterates to achieve the optimal $O(1/T)$ error.
\end{theorem}

\begin{theorem}
\TheoremName{FinalIterateLowerBoundLipschitz}
For any $T$, there exists a convex function $f_T : \cX \rightarrow \bR$,
where $\cX$ is the unit Euclidean ball in $\bR^T$,
such that $f_T$ is $1$-Lipschitz, and satisfies the following.
Suppose that \Algorithm{SGD} is executed
from the initial point $x_1 = 0$ with step sizes $\eta_t = 1/\sqrt{t}$.
Let $x^* = \argmin_{x \in \cX} f_T(x)$.
Then
\begin{equation}
\EquationName{LastIterateLarge}
f_T(x_T) - f_T(x^*) ~\geq~ \frac{\log T}{32 \sqrt{T}}.
\end{equation}
More generally, any weighted average $\bar{x}$ of the last $k$ iterates has
\begin{equation}
\EquationName{EveryIterateLarge}
f_T(\bar{x}) - f_T(x^*) ~\geq~ \frac{\ln(T) - \ln(k)}{32 \sqrt{T}}.
\end{equation}
Furthermore, the value of $f$ strictly monotonically \emph{increases} for the first $T$ iterations:
\begin{equation}
\EquationName{Monotone}
f(x_{i+1}) ~\geq~ f(x_{i}) + \frac{1}{32 \sqrt{T} (T-i+1)}
    \qquad\forall i \in [T].
\end{equation}
\end{theorem}

\begin{remark}
In order to incur a $\log T$ factor in the error of the $T\th$ iterate,
\iffull
	\Theorem{FinalIterateLowerBound} and
    \Theorem{FinalIterateLowerBoundLipschitz}
\fi
constructs a function $f_T$ parameterized by $T$.
It is also possible to create a single function $f$,
\emph{independent} of $T$, 
which incurs the $\log T$ factor for infinitely many $T$.
This is described in \Remark{InfiniteDimension}.
\end{remark}

\iffull
	\subsection{High probability upper bound for suffix averaging}
\fi

Interestingly, our proof of \Theorem{FinalIterateHighProbability} requires understanding the suffix average.
(In fact this connection is implicit in \cite{ShamirZhang}).
Hence, en route, we prove the following high probability bound on the error of the average of the last half of the iterates of SGD.
\begin{theorem}
\TheoremName{SuffixAverageHighProbability}
Suppose $f$ is $1$-strongly convex and $1$-Lipschitz. Consider running \Algorithm{SGD} for $T$ iterations with step size $\eta_t = 1/t$. Let $x^* = \argmin_{x \in \cX} f(x)$. Then,
with probability at least $1-\delta$,
\[
f \bigg (\frac{1}{T/2 + 1}\sum_{t=T/2}^T x_t \bigg) - f(x^*) ~\leq~ O\bigg(\frac{ \log(1/\delta)}{T} \bigg).
\]
\end{theorem}

\begin{remark}
This upper bound is optimal.
Indeed, \Appendix{lb_delta} shows that the error is $\Omega(\log(1/\delta) / T)$ even for the one-dimensional function $f(x) = x^2/2$.
\end{remark}

\Theorem{SuffixAverageHighProbability} is an improvement over the $O\big( \log(\log(T)/\delta)/T \big)$ bounds
independently proven by \citet{Rakhlin} (for suffix averaging) and \citet{HK14} (for EpochGD).
Once again, we defer the statement of the theorem for general strongly-convex and Lipschitz parameters to \Appendix{Generalizations}.

%% file: techniques.tex
\section{Techniques}
\SectionName{Techniques}

\textbf{Final iterate.}
When analyzing gradient descent, it simplifies matters greatly
to consider the \emph{expected} error.
This is because the effect of a gradient step is usually bounded by
the subgradient inequality; so by linearity of expectation,
one can plug in the \emph{expected} subgradient, thus eliminating the noise
\cite[\S 6.1]{Bubeck}.

High probability bounds are more difficult.
(Indeed, it is not a priori obvious that the error of the final iterate is tightly concentrated.)
A high probability analysis must somehow control the total noise that accumulates from each noisy subgradient step. 
Fortunately, the accumulated noise forms a zero-mean martingale but unfortunately, the martingale depends on previous iterates in a highly nontrivial manner.
Indeed, suppose $(X_t)$ is the martingale of the accumulated noise and let $V_{t-1} = \expectg{(X_t - X_{t-1})^2}{X_1, \ldots, X_{t-1}}$ be the conditional variance at time $t$.
A significant technical step of our analysis (\Lemma{w_t_upper_bound})
shows that the total conditional variance (TCV) of the accumulated noise exhibits the ``chicken and egg'' phenomenon alluded to in the discussion of \Theorem{FanV2}.
Roughly speaking, we have $\sum_{t=1}^{T} V_{t-1} \leq \alpha X_{T-1} + \beta$ where $\alpha, \beta > 0$ are scalars.
Since Freedman's inequality shows that $X_T \lesssim \sqrt{\sum_{t=1}^T V_T}$,
an inductive argument gives that
$X_T \lesssim \sqrt{\alpha X_{T-1} + \beta} \lesssim \sqrt{\alpha \sqrt{\alpha X_{T-2}+\beta} + \beta} \lesssim \cdots$.
This naive analysis involves invoking Freedman's inequality $T$ times, so a union bound incurs an extra factor $\log T$ in the bound on $X_T$.
This can be improved via a trick \cite{Bartlettetal08}:
by upper-bounding the TCV by a power-of-two (and by $T$), it suffices to invoke Freedman's inequality $\log T$ times,
which only incurs an extra factor $\log \log T$ in the bound on $X_T$.

Notice that this analysis actually shows that $X_t \lesssim \sqrt{\sum_{i=1}^t V_i}$
\emph{for all} $t \leq T$, whereas the original goal was only to control $X_T$.
Any analysis that simultaneously controls all $X_t$, $t \leq T$,
must necessarily incur an extra factor $\log \log T$.
This is a consequence of the Law of the Iterated Logarithm\footnote{Let $X_t \in \{-1,+1\}$ be uniform and i.i.d.\ and $S_T = \sum_{t=1}^T X_t$. The Law of the Iterated Logarithm states that $\limsup_{T} \frac{S_T}{\sqrt{2T\log\log T}} = 1$ a.s.}.
Previous work employs exactly such an analysis \cite{HK14,KT08,Rakhlin} and incurs the $\log \log T$ factor. 
Rakhlin et al.~\cite{Rakhlin} explicitly raise the question of whether this $\log \log T$ factor is necessary.

Our work circumvents this issue by developing a generalization of Freedman's Inequality (\Theorem{FanV2}) to handle martingales of the above form, which ultimately yields optimal high-probability bounds. We are no longer hindered by the Law of the Iterated Logarithm because our variant of Freedman's Inequality does not require us to have fine grained control over the martingale over all times.

Another important tool that we employ is a new bound on the Euclidean distance between the iterates computed by SGD (\Lemma{StochasticDistanceEstimate}).
This is useful because, by the subgradient inequality, the change in the error at different iterations can be bounded using the distance between iterates. 
Various naive approaches yield a bound of the form $\norm{x_a-x_b}^2 \leq \frac{(b-a)^2}{\min \set{a^2,b^2}}$
\iffull
	$\norm{x_a-x_b}^2 \leq \frac{(b-a)^2}{\min \set{a^2,b^2}}$
	(in the strongly convex case).
\fi
We derive a much stronger bound, comparable to $\norm{x_a-x_b}^2 \leq \frac{\abs{b-a}}{\min \set{a^2,b^2}}$.
Naturally, in the stochastic case, there are additional noise terms that contribute to the technical challenge of our analysis.
Nevertheless, this new distance bound could be useful in further understanding non-smooth gradient descent (even in the non-stochastic setting).

As in previous work on the strongly convex case \cite{ShamirZhang}, the error of the suffix average plays a critical role in bounding the error of the final iterate. Therefore, we also need a tight high probability bound on the error of the suffix average.

\paragraph{Suffix averaging.}

To complete the optimal high probability analysis on the final iterate, we need a high probability bound on the suffix average that avoids the $\log \log T$ factor. As in the final iterate setting, the accumulated noise for the suffix average forms a zero-mean martingale, $(X_t)_{T/2}^T$, but now the conditional variance at step $t$ satisfies $V_t \leq \alpha_t V_{t-1} + \beta_t \hat{w}_t \sqrt{V_{t-1}} + \gamma_t$, where $\hat{w}_t$ is a mean-zero random variable and $\alpha_t, \beta_t$ and $\gamma_t$ are constants. In \cite{Rakhlin}, using Freedman's Inequality combined with the trick from \cite{Bartlettetal08}, they obtain a bound on a similar martingale but do so over all time steps and incur a $\log \log T$ factor. However, our goal is only to bound $X_T$ and according to Freedman's Inequality $X_T \lesssim \sqrt{\sum_{t=T/2}^T V_t}$. So, our goal becomes to bound $\sum_{t=T/2}^{T} V_t$. To do so, we develop a probabilistic tool to bound the $t\th$ iterate of a stochastic process that satisfies a recursive dependence on the $(t-1)\th$ iterate similar to the one exhibited by $V_t$.

\newcommand{\recursiveprocess}{
Let $(X_t)_{t=1}^T$ be a stochastic process and let $(\cF_{t})_{t=1}^T$ be a filtration such that $X_t$ is $\cF_{t}$ measurable and $X_t$ is non-negative almost surely. Let $\alpha_t \in [0,1)$ and $\beta_t,\gamma_t \geq 0$ for every $t$. Let $\hat{w}_t$ be a mean-zero random variable conditioned on $\cF_{t}$ such that $\Abs{\hat{w}_t} \leq 1$ almost surely for every $t$. Suppose that $X_{t+1} \leq \alpha_t X_{t} + \beta_t \hat{w}_t \sqrt{X_{t}} + \gamma_t$ for every $t$. Then, the following hold.
\begin{itemize}
\item For every $t$, $\prob{  X_t \geq K \log(1/\delta) } \leq \me \delta$. 
\item More generally, if $\sigma_1, \ldots , \sigma_T\geq0$, then $\prob{  \sum_{t=1}^T \sigma_t X_t \geq K \log(1/\delta) \sum_{t= 1}^T \sigma_t  } \leq \me \delta$,
\end{itemize}
where $K = \max_{1 \leq t \leq T} \left(  \frac{2 \gamma_t}{1 - \alpha_t}, \frac{2 \beta_t^2}{1 - \alpha_t}  \right)$.
}

\begin{theorem}
\TheoremName{RecursiveStochasticProcess}
\recursiveprocess
\end{theorem}

The recursion $X_{t+1} \leq \alpha_t + \beta_t \hat{w}_t \sqrt{X_t} + \gamma_t$ presents two challenges that make it difficult to analyze. Firstly, the fact that it is a non-linear recurrence makes it unclear how one should unwind $X_{t+1}$. Furthermore, unraveling the recurrence introduces many $\hat{w}_t$ terms in a non-trivial way. Interestingly, if we instead consider the moment generating function (MGF) of $X_{t+1}$, then we can derive an analogous recursive MGF relationship which removes this non-linear dependence and removes the $\hat{w}_t$ term. This greatly simplifies the recursion and leads to a surprisingly clean analysis. The proof of \Theorem{RecursiveStochasticProcess} can be found in \Appendix{RecursiveStochasticProcess}.
(The recursive MGF bound which removes the non-linear dependence is by \Claim{RecursiveMGFBound}.)

\paragraph{Deterministic lower bound.}
As mentioned above, a challenge with non-smooth gradient is that the error of the $T\th$ iterate may not monotonically decrease with $T$, even in the deterministic setting.
The full extent of this non-decreasing behavior seems not to have been previously understood.
We develop a technique that forces the error to be monotonically \emph{increasing} for $\Omega(T)$ consecutive iterations.
The idea is as follows.
If GD takes a step in a certain direction, a non-differentiable point can allow the function to suddenly increase in that direction.
If the function were one-dimensional, the next iteration of GD would then be guaranteed to step in the opposite direction, thereby decreasing the function.
However, in higher dimensions, the second gradient step could be nearly orthogonal to the first step, and the function could have yet another non-differentiable point in this second direction.
In sufficiently high dimensions, this behavior can be repeated for many iterations.
The tricky aspect is designing the function to have this behavior while also being convex.
We show that this is possible, leading to the unexpectedly large 
error in the $T\th$ iteration.
We believe that this example illuminates some non-obvious behavior of gradient descent.

%% file: lbshort.tex
\iffull
	\section{Lower bound on error of final iterate, strongly convex case}
\fi
\SectionName{lbshort}

In this section we prove that the final iterate of SGD 
\iffull
	for strongly convex functions
\fi
has error that is suboptimal by a factor $\Omega(\log T)$, even in the non-stochastic case.
Specifically, we define a function $f = f_T$, depending on $T$,
for which the final iterate produced by \Algorithm{SGD} has
$f(x_T) = \Omega(\log(T)/T)$, thereby proving \eqref{eq:LastIterateLargeSC}.
Let $\cX$ be the Euclidean unit ball in $\bR^T$. 
Define $f : \cX \rightarrow \bR$ and $h_i \in \bR^T$ for $i \in [T+1]$ by
\begin{align*}
f(x) &~=~ \max_{i \in [T+1]} H_i(x)
	\qquad\text{where}\qquad H_i(x) = h_i \transpose x + \frac{1}{2} \norm{x}^2 \\
h_{i,j} &~=~
\begin{cases}
a_j   &\quad\text{(if $1 \leq j < i$)} \\
-1    &\quad\text{(if $i = j \leq T$)} \\
0     &\quad\text{(if $i < j \leq T$)}
\end{cases}
\qquad\text{and}\qquad
a_j ~=~ \frac{1}{2(T+1-j)} \qquad\text{(for $j \in [T]$)}.
\end{align*}
It is easy to see that $f$ is $1$-strongly convex due to the $\frac{1}{2} \norm{x}^2$ term.
Furthermore $f$ is $3$-Lipschitz over $\cX$ because 
$\norm{\Grad H_i(x)} \leq \norm{h_i}+1$ and
$\norm{h_i}^2 \leq 1 + \frac{1}{4} \sum_{j=1}^{T} \frac{1}{(T-j)^2} < 1 + \frac{1}{2}$.
Finally, the minimum value of $f$ over $\cX$ is non-positive because $f(0)=0$.

\paragraph{Subgradient oracle.}
In order to execute \Algorithm{SGD} on $f$ we must specify a subgradient oracle.
First, we require the following claim,
which follows from standard facts in convex analysis \cite[Theorem 4.4.2]{HUL}.

\begin{claim}
\ClaimName{SetOfSubgradients}
$\partial f(x)$ is the convex hull of 
$\setst{ h_i + x }{ i \in \cI(x) }$, where
$\cI(x) = \setst{ i }{ H_i(x) = f(x) }$.
\end{claim}

\noindent Our subgradient oracle is non-stochastic:
given $x$, it simply returns $h_{i'}+x$ where $i' = \min \cI(x)$.

\paragraph{Explicit description of iterates.}
Next we will explicitly describe the iterates produced by executing \Algorithm{SGD} on $f$.
Define the points $z_t \in \bR^T$ for $t \in [T+1]$ by $z_1=0$ and
\[
z_{t,j} ~=~ \begin{cases}
\displaystyle \frac{1-(t-j-1)a_j}{t-1} &\quad\text{(if $1 \leq j < t$)} \\
0                                      &\quad\text{(if $t \leq j \leq T$)}.
\end{cases}
\qquad\text{(for $t>1$).}
\]
We will show inductively that these are precisely the first $T$ iterates produced by \Algorithm{SGD}
when using the subgradient oracle defined above.
The following claim is easy to verify from the definition of $z_t$.

\begin{claim} \mbox{}
\ClaimName{CombinedClaimAboutZ}
\begin{itemize}
\item For $t \in [T+1]$, $z_t$ is non-negative.
	  In particular, $z_{t,j} \geq \frac{1}{2 (t-1)}$ for $j < t$ and $z_{t,j} = 0$ for $j \geq t$.
\item $\norm{z_1} = 0$ and $\norm{z_t}^2 \leq \frac{1}{t-1}$ for $t>1$. Thus $z_t \in \cX$ for all $t \in [T+1]$.
\end{itemize}
\end{claim}

\noindent
The ``triangular shape'' of the $h_i$ vectors allows us to determine the
value and subdifferential at $z_t$.

\begin{claim}
\ClaimName{OracleSubgradientSC}
$f(z_t) = H_t(z_t)$ for all $t \in [T+1]$.
The subgradient oracle for $f$ at $z_t$ returns the vector $h_t+z_t$.
\end{claim}

\begin{proof}
We claim that $h_t \transpose z_t = h_i \transpose z_t$ for all $i>t$.
By definition, $z_t$ is supported on its first $t-1$ coordinates.
However, $h_t$ and $h_i$ agree on the first $t-1$ coordinates (for $i>t$).
This proves the first part of the claim.

Next we claim that $z_t \transpose h_t > z_t \transpose h_i$ for all $1 \leq i < t$.
This also follows from the definition of $z_t$ and $h_i$:
\begin{align*}
z_t \transpose (h_t - h_i)
    ~=~ \sum_{j=1}^{t-1} z_{t,j} (h_{t,j} - h_{i,j})
    ~=~ \sum_{j=i}^{t-1} z_{t,j} (h_{t,j} - h_{i,j})
    ~=~ z_{t,i} (a_i + 1) + \sum_{j=i+1}^{t-1} z_{t,j} a_j
    ~>~ 0.
\end{align*}

These two claims imply that $H_t(z_t) \geq H_i(z_t)$ for all $i \in [T+1]$,
and therefore $f(z_t) = H_t(z_t)$.
Moreover $\cI(z_t) = \setst{ i }{ H_i(z_t) = f(z_t) } = \set{t,\ldots,T+1}$.
Thus, when evaluating the subgradient oracle at the vector $z_t$, it returns the vector $h_t+z_t$.
\end{proof}

Since the subgradient returned at $z_t$ is determined by \Claim{OracleSubgradientSC},
and the next iterate of SGD arises from a step in the opposite direction,
a straightforward induction proof allows us to show the following lemma.
A detailed proof is in \Appendix{LowerBound}.

\begin{lemma}
\LemmaName{XisZSC}
\iffull
	For the function $f$ constructed in this section, the
\fi
vector $x_t$ in \Algorithm{SGD} equals $z_t$, for every $t \in [T+1]$.
\end{lemma}

The value of the final iterate is easy to determine from
\Lemma{XisZSC} and \Claim{OracleSubgradientSC}:
\[
f(x_{T+1}) ~=~ f(z_{T+1}) ~=~ H_{T+1}(z_{T+1})
  ~\geq~ \sum_{j=1}^T h_{T+1,j} \cdot z_{T+1,j}
  ~\geq~ \sum_{j=1}^T \frac{1}{2(T+1-j)} \cdot \frac{1}{2T}
  ~>~ \frac{\log T}{4 T}.
\]
(Here the second inequality uses \Claim{CombinedClaimAboutZ}.)
This proves \eqref{eq:LastIterateLargeSC}.
A small modification of the last calculation proves \eqref{eq:EveryIterateLargeSC};
details may be found in \Claim{LBForSuffix}.
This completes the proof of \Theorem{FinalIterateLowerBound}.

\iffull
\section{Lower bound on error of final iterate, Lipschitz case}
\SectionName{lbshortLipschitz}

In this section we prove a lower bound result  for Lipschitz functions analogous to those in \Section{lbshort}.
Specifically, we define a function $f = f_T$, depending on $T$,
for which the final iterate produced by \Algorithm{SGD} has
$f(x_T) = \Omega(\log(T)/\sqrt{T})$, thereby proving \eqref{eq:LastIterateLarge}.
Throughout this section we will assume that
$\eta_t = \frac{c}{\sqrt{t}}$ for $c \geq 1$.

The function $f$ is defined as follows.
As before, $\cX$ denotes the Euclidean unit ball in $\bR^T$.
For $i \in [T]$, define the positive scalar parameters 
$$
a_i ~=~ \frac{1}{8c(T-i+1)} \hspace{3cm}
b_i ~=~ \frac{\sqrt{i}}{2c\sqrt{T}}.
$$
Define $f : \cX \rightarrow \bR$ and $h_i \in \bR^T$ for $i \in [T+1]$ by
\begin{align*}
f(x) ~=~ \max_{i \in [T+1]} h_i \transpose x
	\qquad\text{where}\qquad
h_{i,j} ~=~
\begin{cases}
a_j   &\quad\text{(if $1 \leq j < i$)} \\
-b_i  &\quad\text{(if $i = j \leq T$)} \\
0     &\quad\text{(if $i < j \leq T$)}
\end{cases}.
\end{align*}
Note that $f$ is $1$-Lipschitz over $\cX$ because 
$$
\norm{h_i}^2
	~\leq~ \sum_{j=1}^{T} a_j^2 + b_T^2
    ~=~ \frac{1}{64c^2} \sum_{j=1}^{T} \frac{1}{j^2} + \frac{1}{4c^2}
    ~<~ \frac{1}{2}.
$$
Also, the minimum value of $f$ over $\cX$ is non-positive because $f(0)=0$.

\paragraph{Subgradient oracle.}
In order to execute \Algorithm{SGD} on $f$ we must specify a subgradient oracle.
Similar to \Claim{SetOfSubgradientsLipschitz}, \cite[Theorem 4.4.2]{HUL} implies

\begin{claim}
\ClaimName{SetOfSubgradientsLipschitz}
$\partial f(x)$ is the convex hull of 
$\setst{ h_i }{ i \in \cI(x) }$, where
$\cI(x) = \setst{ i }{ h_i \transpose x = f(x) }$.
\end{claim}

\noindent Our subgradient oracle is as follows:
given $x$, it simply returns $h_{i'}+x$ where $i' = \min \cI(x)$.

\paragraph{Explicit description of iterates.}
Next we will explicitly describe the iterates produced by executing \Algorithm{SGD} on $f$.
Define the points $z_t \in \bR^T$ for $t \in [T+1]$ by $z_1=0$ and
\[
z_{t,j} ~=~ \begin{cases}
\displaystyle c \Bigg(\frac{b_j}{\sqrt{j}} \,-\, a_j \sum_{k=j+1}^{t-1} \frac{1}{\sqrt{k}}\Bigg)
    &\quad\text{(if $1 \leq j<t$)} \\
0   &\quad\text{(if $t \leq j \leq T$)}.
\end{cases}
\qquad\text{(for $t>1$).}
\]
We will show inductively that these are precisely the first $T$ iterates produced by \Algorithm{SGD}
when using the subgradient oracle defined above.

\begin{claim}
\ClaimName{ZNonNegative}
For $t \in [T+1]$, $z_t$ is non-negative.
In particular, $z_{t,j} \geq \frac{1}{4 \sqrt{T}}$ for $j < t$ and $z_{t,j} = 0$ for $j \geq t$.
\end{claim}
\begin{proof}
By definition, $z_{t,j}=0$ for all $j \geq t$.
For $j < t$,
\begin{align*}
    z_{t,j}
        &~=~ c \Bigg(\frac{b_j}{\sqrt{j}} - a_j \sum_{k=j+1}^{t-1} \frac{1}{\sqrt{k}}\Bigg) \\
        &~=~ c \Bigg(\frac{1}{2 c \sqrt{T}} - \frac{1}{8c(T-j+1)} \sum_{k=j+1}^{t-1} \frac{1}{\sqrt{k}}\Bigg)
            \quad\text{(by definition of $a_j$ and $b_j$)}\\
        &~\geq~ \frac{1}{2\sqrt{T}} - \frac{1}{4(T-j+1)} \frac{t-1-j}{\sqrt{t-1}}
            \quad\text{(by \Claim{SumRecipSqrt})}\\
        &~\geq~ \frac{1}{2\sqrt{T}} - \frac{1}{4\sqrt{T}}
            \quad\text{(by \Claim{WeirdInequality})}\\
        &~=~ \frac{1}{4\sqrt{T}}. \qedhere
\end{align*}
\end{proof}

\begin{claim}
\ClaimName{ZinX}
$z_{t,j} \leq 1/\sqrt{T}$ for all $j$.
In particular, $z_t \in \cX$ (the unit ball in $\bR^T$).
\end{claim}
\begin{proof}
We have $z_{t,j}=0$ for all $j \geq t$, and for $j < t$, we have
$$
    z_{t,j}
        ~=~ c \Bigg( \frac{b_j}{\sqrt{j}} - a_j \sum_{k=j+1}^{t} \frac{1}{\sqrt{k}}\Bigg)
        ~\leq~ c \frac{b_j}{\sqrt{j}}
        ~=~ \frac{1}{2\sqrt{T}}.
$$
Since \Claim{ZNonNegative} shows that $z_t \geq 0$, we have $\norm{z_t} \leq 1$,
and therefore $z_t \in \cX$.
\end{proof}

\noindent
The ``triangular shape'' of the $h_i$ vectors allows us to determine the
value and subdifferential at $z_t$.

\begin{claim}
\ClaimName{OracleSubgradient}
$f(z_t) = h_t \transpose z_t$ for all $t \in [T+1]$.
The subgradient oracle for $f$ at $z_t$ returns the vector $h_t$.
\end{claim}

\begin{proof}
We claim that $h_t \transpose z_t = h_i \transpose z_t$ for all $i>t$.
By definition, $z_t$ is supported on its first $t-1$ coordinates.
However, $h_t$ and $h_i$ agree on the first $t-1$ coordinates (for $i>t$).
This proves the first part of the claim.

Next we claim that $z_t \transpose h_t > z_t \transpose h_i$ for all $1 \leq i < t$.
This also follows from the definition of $z_t$ and $h_i$:
\begin{align*}
z_t \transpose (h_t - h_i)
    &~=~ \sum_{j=1}^{t-1} z_{t,j} (h_{t,j} - h_{i,j})
        \quad\text{($z_t$ is supported on first $t-1$ coordinates)}\\
    &~=~ \sum_{j=i}^{t-1} z_{t,j} (h_{t,j} - h_{i,j})
        \quad\text{($h_i$ and $h_t$ agree on first $i-1$ coordinates)}\\
    &~=~ z_{t,i} (a_i + b_i) + \sum_{j=i+1}^{t-1} z_{t,j} a_j \\
    &~>~ 0.
\end{align*}

These two claims imply that $h_t \transpose z_t \geq h_i \transpose z_t$ for all $i \in [T+1]$,
and therefore $f(z_t) = h_t \transpose z_t$.
Moreover $\cI(z_t) = \setst{ i }{ h_i \transpose z_t = f(z_t) } = \set{t,\ldots,T+1}$.
Thus, when evaluating the subgradient oracle at the vector $z_t$, it returns the vector $h_t$.
\end{proof}

Since the subgradient returned at $z_t$ is determined by \Claim{OracleSubgradient},
and the next iterate of SGD arises from a step in the opposite direction,
a straightforward induction proof allows us to show the following lemma.

\begin{lemma}
\LemmaName{XisZ}
For the function $f$ constructed in this section, the
vector $x_t$ in \Algorithm{SGD} equals $z_t$, for every $t \in [T+1]$.
\end{lemma}

\begin{proof}
The proof is by induction.
By definition $x_1=0$ and $z_1=0$, establishing the base case.

So assume $z_t = x_t$ for $t \leq T$; we will prove that $z_{t+1} = x_{t+1}$.
Recall that \Algorithm{SGD} sets
$y_{t+1} = x_t - \eta_t g_t$,
and that $\eta_t = \frac{c}{\sqrt{t}}$.
By the inductive hypothesis, $x_t = z_t$.
By \Claim{OracleSubgradient}, the algorithm uses the subgradient $g_t = h_t$.
Thus, 
\begin{align*}
y_{t+1,j} 
    &~=~ z_{t,j} - \frac{c}{\sqrt{t}} h_{t,j} \\
    &~=~ c
        \left\lbrace
        \begin{array}{ll}
        \frac{b_j}{\sqrt{j}} - a_j \sum_{k=j+1}^{t-1} \frac{1}{\sqrt{k}}
            &\quad\text{(for $1 \leq j<t$)} \\
        0   &\quad\text{(for $j \geq t$)}
        \end{array}
        \right\rbrace
    ~-~ \frac{c}{\sqrt{t}}
        \left\lbrace
        \begin{array}{ll}
        a_j     &\quad\text{(for $1 \leq j<t$)} \\
        -b_t    &\quad\text{(for $j=t$)} \\
        0       &\quad\text{(for $j>t$)}
        \end{array}
        \right\rbrace \\
    &~=~ c
        \left\lbrace
        \begin{array}{ll}
        \frac{b_j}{\sqrt{j}} - a_j \sum_{k=j+1}^{t} \frac{1}{\sqrt{k}}
            &\quad\text{(for $j<t$)} \\
        \frac{b_t}{\sqrt{t}}
            &\quad\text{(for $j=t$)} \\
        0   &\quad\text{(for $j>t$)}
        \end{array}
        \right\rbrace
\end{align*}
So $y_{t+1} = z_{t+1}$.
Since $x_{t+1} = \Pi_{\cB_T}(y_{t+1})$ by definition,
and $y_{t+1} \in \cX$ by \Claim{ZinX},
we have $x_{t+1} = y_{t+1} = z_{t+1}$.
\end{proof}

The value of the final iterate is easy to determine from
\Lemma{XisZSC} and \Claim{OracleSubgradientSC}:
\[
f(x_{T+1}) ~=~ f(z_{T+1}) ~=~ h_{T+1} \transpose z_{T+1}
  ~=~ \sum_{j=1}^T h_{T+1,j} \cdot z_{T+1,j}
  ~\geq~ \sum_{j=1}^T \frac{1}{8c(T+1-j)} \cdot \frac{1}{4 \sqrt{T}}
  ~>~ \frac{\log T}{32c \sqrt{T}}.
\]
(Here the second inequality uses \Claim{ZNonNegative}.)
This proves \eqref{eq:LastIterateLarge}.
A small modification of the last calculation proves \eqref{eq:EveryIterateLarge};
details may be found in \Claim{LBForSuffixLipschitz}.
The proof of \eqref{eq:Monotone}
may be found in \Subsection{LBMonotone}.
This completes the proof of \Theorem{FinalIterateLowerBoundLipschitz}.

\fi

%% file: ubshortnew.tex
\iffull
	\section{Upper bound on error of final iterate, strongly convex case}
\fi

\SectionName{ubshort}
We now turn to the proof of the upper bound on the error of the final iterate of SGD, in the case where $f$ is 1-strongly convex and 1-Lipschitz (\Theorem{FinalIterateHighProbability}). Recall that the step size used by \Algorithm{SGD} in this case is $\eta_t = 1/t$. 
We will write $\hat{g}_t = g_t - \hat{z}_t$, where $\hat{g}_t$ is the vector returned by the oracle at the point $x_t$, $g_t \in \partial f(x_t)$, and $\hat{z}_t$ is the noise.
Let $\cF_t = \sigma(\hat{z}_1, \ldots, \hat{z}_t)$ be the $\sigma$-algebra generated by the first $t$ steps of SGD.
Finally, recall that $\norm{\hat{z}_t} \leq 1$ and $\expectg{\hat{z}_t}{\cF_{t-1}} = 0$. 

We begin with the following lemma which can be inferred from the proof of Theorem~1 in \citet{ShamirZhang}.
For completeness, we provide a proof in \Appendix{ub_omitted}.
\begin{lemma}
\LemmaName{ShamirZhangHalf}
Let $f$ be 1-strongly convex and 1-Lipschitz.
Suppose that we run SGD (\Algorithm{SGD}) with step sizes $\eta_t = 1/t$.
Then
\begin{align*}
f(x_T)
& ~\leq~ 
\underbrace{\frac{1}{T/2 + 1} \sum_{t=T/2}^T f(x_t)}_{\text{suffix average}}
~+~ \underbrace{\sum_{k=1}^{T/2} \frac{1}{k(k+1)} \sum_{t=T-k}^T \inner{\hat{z}_t}{x_t - x_{T-k}}}_{\text{$Z_T$, the noise term}}
~+~ O\left( \frac{\log T}{T} \right).
\end{align*}
\end{lemma}
\Lemma{ShamirZhangHalf} asserts that the error of the last iterate is upper bounded by the sum of the error of the suffix average and some noise terms (up to the additive $O(\log T / T)$ term).
Thus, it remains to show that the error due to the suffix average is small with high probability (\Theorem{SuffixAverageHighProbability}) and the noise terms are small.
We defer the proof of \Theorem{SuffixAverageHighProbability} to \Subsection{SuffixAverage}.
By changing the order of summation, we can write $Z_T = \sum_{t=T/2}^T \inner{\hat{z}_t}{w_t}$ where
\[
w_t = \sum_{j=T/2}^t  \alpha_j (x_t - x_j)
\qquad\text{and}\qquad
\alpha_j = \frac{1}{(T-j)(T-j+1)}.
\]

The main technical difficulty is to show that $Z_T$ is small with high probability.
Formally, we prove the following lemma,
whose proof is outlined in \Subsection{BoundingTheNoise}.

\begin{lemma}
\LemmaName{FinalIterateNoise}
$Z_T \leq O\left( \frac{\log(T) \log(1/\delta)}{T} \right)$ with probability at least $1 - \delta$.
\end{lemma}

\noindent Given \Theorem{SuffixAverageHighProbability} and \Lemma{FinalIterateNoise}, the proof of \Theorem{FinalIterateHighProbability} is immediate.


\subsection{Bounding the noise}
\SubsectionName{BoundingTheNoise}

The main technical difficulty in the proof is to understand the noise term, which we have denoted by $Z_T$.
Notice that $Z_T$ is a sum of a martingale difference sequence.
The natural starting point is to better understand the TCV of $Z_T$ (i.e. $\sum_{t= T/2}^T \norm{w_t}^2$). We we will see that $\sum_{t=T/2}^T\norm{w_t}^2$ is bounded by a linear transformation of $Z_T$. This ``chicken and egg'' relationship inspires us to derive a new probabilistic tool (generalizing Freedman's Inequality) to disentangle the total conditional variance from the martingale.

The main challenge in analyzing $\norm{w_t}$ is precisely analyzing the distance $\norm{x_t - x_j}$ between SGD iterates.
A loose bound of
$\norm{x_t-x_j}^2 \lesssim (t-j) \sum_{i=j}^t \frac{\norm{\hat{g}_i}^2}{i^2}$ follows easily from Jensen's Inequality.
We prove the following tighter bound, which may be of independent interest.
The proof is in \Appendix{ub_omitted}.

\newcommand{\distanceestimate}{
Suppose $f$ is 1-Lipschitz and 1-strongly convex. Suppose we run \Algorithm{SGD} for $T$ iterations with step sizes $\eta_t = 1/t$. Let $a < b$. Then,
\[
\norm{x_a - x_b }^2 \leq \sum_{i=a}^{b-1} \frac{\norm{\hat{g}_i}^2}{i^2} + 2 \sum_{i=a}^{b-1} \frac{\big (f(x_a) - f(x_i) \big)}{{i}} + 2\sum_{i=a}^{b-1} \frac{\inner{\hat{z}_i}{x_i - x_a}}{{i}}.
\]
}

\begin{lemma}
\LemmaName{StochasticDistanceEstimate}
\distanceestimate
\end{lemma}

Using \Lemma{StochasticDistanceEstimate} and some delicate calculations we obtain the following upper bound on $\sum_{t=T/2}^T \norm{w_t}^2$, revealing the surprisingly intricate relationship between $Z_T$ (the martingale) and $\sum_{t=T/2}^T\norm{w_t}^2$
(its TCV).
This is the main technical step that inspired our probabilistic tool (the generalized Freedman's Inequality).

\begin{lemma}[Main Technical Lemma]
\LemmaName{w_t_upper_bound}
There exists positive values $R_1 = O\left( \frac{\log^2 T}{T^2} \right)$,
$R_2 = O\left( \frac{\log T}{T} \right)$,
$C_t = O(\log T)$,
$A_t = O\left( \frac{\log T}{T^2} \right)$ such that
\begin{equation}
\EquationName{MainTechnicalLemma}
\sum_{t=T/2}^T \norm{w_t}^2 \leq R_1 + R_2 \norm{x_{T/2} - x^*}^2
+ \underbrace{\sum_{t=T/2}^{T-1} \frac{C_t}{t} \inner{\hat{z}_t}{w_t}}_{
    	\approx O(\log T/T) Z_T
    }
    + \sum_{t=T/2}^{T-1} \inner{\hat{z}_t}{A_t (x_t - x^*)}.
\end{equation}
\end{lemma}

This bound is mysterious in that the left-hand side is an upper bound on the total conditional variance of $Z_T$, whereas the right-hand side essentially contains a scaled version of $Z_T$ itself.
This is the ``chicken and egg phenomenon'' alluded to in \Section{Techniques},
and it poses another one of the main challenges of bounding $Z_T$. This bound inspires our main probabilistic tool,
which we restate for convenience here.

\repeatclaimwithname{\Theorem{FanV2}}{Generalized Freedman}{\fantwo}

\vspace{-6pt}

In order to apply \Theorem{FanV2}, we need to refine \Lemma{w_t_upper_bound} to replace the terms $\norm{x_{T/2} - x^*}^2$ and $\sum_{t=T/2}^{T-1}\inner{\hat{z}_t}{A_t(x_t - x^*)}$ with sufficient high probability upper bounds.  In \cite{Rakhlin}, they showed that $\norm{x_t - x^*}^2 \leq O(\log \log(T) / T)$ for all $\frac{T}{2} \leq t \leq T$ simultaneously, with high probability, so using that would give a slightly suboptimal result.
In contrast, our analysis only needs a high probability bound on $\norm{x_{T/2} - x^*}^2$ and $\sum_{t=T/2}^T A_t \norm{x_t - x^*}^2$; this allows us to avoid a $\log \log T$ factor here. Indeed, we have

\begin{theorem}
\TheoremName{IteratesHighProb}
Both of the following hold:
\begin{itemize}
\item For all $t \geq 2$, $\norm{x_t-x^*}^2 \leq O\left(\log(1/\delta)/t\right)$ with probability $1-\delta$, and
\item Let $\sigma_t \geq 0$ for $t= 2, \ldots , T$. Then, $\sum_{t=2}^T \sigma_t \norm{x_t - x^*} = O\left ( \sum_{t=2}^T \frac{\sigma_t}{t} \log(1/\delta) \right )$ w.p.\ $1-\delta$.  
\end{itemize}
\end{theorem}

The proof of \Theorem{IteratesHighProb}, in \Subsection{IteratesHighProb}, uses our tool for bounding recursive stochastic processes (\Theorem{RecursiveStochasticProcess}). Therefore, we need to expose a recursive relationship between $\norm{x_{t+1} - x^*}^2$ and $\norm{x_{t} - x^*}^2$ that satisfies the conditions of \Theorem{RecursiveStochasticProcess}.
Interestingly, \Theorem{IteratesHighProb} is also the main ingredient in the analysis of the error of the suffix average (see \Subsection{SuffixAverage}). We now have enough to give our refined version of \Lemma{w_t_upper_bound}, which is now in a format usable by Freedman's Inequality.

\begin{lemma}
\LemmaName{w_t_upper_bound_hp}
For every $\delta > 0$ there exists positive values $R = O\left( \frac{\log^2 T \log(1/\delta)}{T^2} \right)$, $C_t = O\left ( \log T \right)$ such that $\sum_{t=T/2}^T \norm{w_t}^2 \leq R +
 \sum_{t=T/2}^{T-1} \frac{C_t}{t} \inner{\hat{z}_t}{w_t},$ with probability at least $1 - \delta$. 
\end{lemma}
\begin{proof}
The lemma essentially follows from combining our bounds in \Theorem{IteratesHighProb} with an easy corollary of Freedman's Inequality (\Corollary{RelateVarianceToMartingale}) which states that a high probability bound of $M$ on the TCV of a martingale implies a high probability bound of $\sqrt{M}$ on the martingale.

Let $R_1$, $R_2$, $C_t$, and $A_t$ be as in \Lemma{w_t_upper_bound}, and consider the resulting upper bound on $\sum_{t=T/2}^T \norm{w_t}^2$. The first claim in \Theorem{IteratesHighProb} gives $R_2\norm{x_{T/2}-x^*}^2 = O \left ( \frac{\log^2T\log(1/\delta)}{T^2} \right )$ because $R_2 = O \left ( \log T  / T \right)$.

By the second claim in \Theorem{IteratesHighProb}, we have $\sum_{t = T/2}^{T-1} A_t^2 \norm{x_t - x^*}^2 = O \left ( \frac{\log^2 T}{T^4} \log(1/\delta) \right )$ with probability at least $1 - \delta$ because each $A_t = O \left ( \frac{\log T}{T^2} \right )$. Hence, we have derived a high probability bound on the total conditional variance of $\sum_{t=T/2}^T\inner{\hat{z}_t}{A_t (x_t - x^*)}$. Therefore, we turn this into a high probability bound on the martingale itself by applying \Corollary{RelateVarianceToMartingale} and obtain $\sum_{t=  T/2}^{T-1} \inner{\hat{z}_t}{A_t (x_t - x^*)} = O\left( \frac{\log^2 T \log(1/\delta)}{T^2} \right)$ with probability at least $1- \delta$.
\end{proof}

Now that we have derived an upper bound on the total conditional variance of $Z_T$ in the form required by our Generalized Freedman Inequality (\Theorem{FanV2}), we are  finally ready to prove \Lemma{FinalIterateNoise} (our high probability upper bound on the noise, $Z_T$).

\begin{proofof}{\Lemma{FinalIterateNoise}}
We have demonstrated that $Z_T$ satisfies the ``Chicken and Egg'' phenomenon with high probability. Translating this into a high probability upper bound on the martingale $Z_T$ itself is a corollary of \Theorem{FanV2}. 

Indeed, consider a filtration $\{\cF_t \}_{t=T/2}^{T}$. Let $d_t = \inner{a_t}{b_t}$ define a martingale difference sequence where $\norm{a_t}\leq 1$ and $\expectg{a_t}{\cF_{t-1}} = 0$. Suppose there are positive values, $R$, $\alpha_t$, such that $\max_{t=T/2}^{T} \{ \alpha_t \} = O\left(\sqrt{R} \right)$ and $\sum_{t=T/2}^{T} \norm{b_t}^2 \leq \sum_{t=T/2}^{T} \alpha_t d_t + R\log(1/\delta)$ with probability at least $1-\delta$. Then, \Corollary{RelateChickenAndEggToMartingale} bounds the martingale at time step $T$ by $\sqrt{R} \log(1/\delta)$ with high probability. 

Observe that \Lemma{w_t_upper_bound_hp} allows us to apply \Corollary{RelateChickenAndEggToMartingale} with $a_t = \hat{z}_t$, $b_t = w_t$, $\alpha_t = (C_t/t)$ for $t = T/2, \ldots, T-1$, $\alpha_T = 0$, $\max_{t=T/2}^T \{ \alpha_t\} = O\left (\log T/T\right )$, and $R = O\left ( \log^2 T / T^2 \right )$ to prove \Lemma{FinalIterateNoise}. 
\end{proofof}


\subsection{High Probability Bounds on Squared Distances to \texorpdfstring{$x^*$}{x*}}
\SubsectionName{IteratesHighProb}

In this section, we prove \Theorem{IteratesHighProb}. 
We begin with the following claim which can be extracted from \cite{Rakhlin}.
\begin{claim}[{\cite[Proof of Lemma~6]{Rakhlin}}]
\ClaimName{ToyProblem}
Suppose $f$ is $1$-strongly-convex and $1$-Lipschitz.
Define $Y_t = t \norm{x_{t+1} - x^*}^2$ and $U_{t} = \inner{\hat{z}_{t+1}}{x_{t+1} - x^*}/ \norm{x_{t+1} - x^*}_2$. Then
\[
Y_t ~\leq~ \bigg ( \frac{t-2}{t-1} \bigg ) Y_{t-1} + 2 \cdot U_{t-1}\sqrt{\frac{Y_{t-1}}{t-1}} + \frac{4}{t}.
\]
\end{claim}
This claim exposes a recursive relationship between $\norm{x_{t+1} - x^*}^2$ and $\norm{x_t - x^*}^2$ and inspires our probabilistic tool for recursive stochastic processes (\Theorem{RecursiveStochasticProcess}). We prove \Theorem{IteratesHighProb} using this tool:

\begin{proofof}{\Theorem{IteratesHighProb}}
Consider the stochastic process $(Y_t)_{t=1}^{T-1}$ where $Y_t$ is as defined by \Claim{ToyProblem}. Note that $Y_t$ satisfies the conditions of \Theorem{RecursiveStochasticProcess} with $X_t = Y_t$, $\hat{w}_t = U_t$, $\alpha_t = \frac{t-2}{t-1} = 1 - 1/(t-1)$, $\beta_t = 2/\sqrt{t-1}$, and $\gamma_t = 4/t$. Observe that $U_t$ is a $\cF_{t+1}$ measurable random variable which is mean zero conditioned on $\cF_{t}$ Furthermore, note that $\Abs{U_t} \leq 1$ with probability 1 because $\norm{\hat{z}_{t+1}} \leq 1$ with probability 1. Furthermore, it is easy to check that $\max_{1 \leq t \leq T} \left ( \frac{2\gamma_t}{1 - \alpha_t} , \frac{2 \beta^2}{1 - \alpha_t} \right ) = 8$ with the above setup. So, we may apply \Theorem{RecursiveStochasticProcess} to obtain:
\begin{itemize}
\item For every $t = 1 , \ldots T-1$, $\prob{Y_t \geq 8 \log(1/\delta)} \leq \me \delta.$
\item Let $\sigma_t'\geq 0$ for $t = 1, \ldots , T-1$. Then, $\prob{  \sum_{t=1}^{T-1} \sigma_t' Y_t \geq 8 \sum_{t=1}^{T-1} \sigma_t'  } \leq \me\delta$.
\end{itemize}
Recalling that $Y_t = t \norm{ x_{t+1} - x^* }^2$ and setting $\sigma_t' = \sigma_t/t$ proves \Theorem{IteratesHighProb}.
\end{proofof}

\subsection{Upper Bound on Error of Suffix Averaging}
\SubsectionName{SuffixAverage}

To complete the proof of the final iterate upper bound (\Theorem{FinalIterateHighProbability}), it still remains to prove the suffix averaging upper bound (\Theorem{SuffixAverageHighProbability}).
In this section, we prove this result as a corollary of the high probability bounds on $\norm{x_t - x^*}^2$ that we obtained in the previous subsection.

\begin{proofof}{\Theorem{SuffixAverageHighProbability}}
By \Lemma{StandardSGDAnalysis} with $w = x^*$ we have
\begin{align}
\EquationName{StandardSGDSuffix}
\sum_{t=T/2}^T \left [ f(x_t) - f(x^*) \right ] ~\leq~ \underbrace{\frac{1}{2}\sum_{t = T/2}^T \eta_t \norm{\hat{g}_t}^2}_{(a)} + \underbrace{\frac{1}{2\eta_{T/2}}\norm{x_{T/2} - x^*}^2}_{(b)} + \underbrace{\sum_{t=T/2}^T \inner{\hat{z}_t}{x_t - x^*}}_{(c)}. 
\end{align}
It suffices to bound the right hand side of \eqref{eq:StandardSGDSuffix} by $O(\log(1/\delta))$ with probability at least $1- \delta$. Indeed, bounding $\norm{\hat{g}_t}^2$ by 4, (a) in \eqref{eq:StandardSGDSuffix} is bounded by $O(1)$. Term (b) is bounded by $O(\log(1/\delta))$ by \Theorem{IteratesHighProb}. 

It remains to bound (c). \Theorem{IteratesHighProb} implies $\sum_{t= T/2}^T \norm{x_t - x^*}^2 = O(\log(1/\delta))$ with probability at least $1- \delta$. Therefore, \Corollary{RelateVarianceToMartingale} shows that (c) is at most $O(\log(1/\delta))$ with probability at least $1 - \delta$. 
\end{proofof}

\iffull
\section{Upper bound on error of final iterate, Lipschitz case: Proof Sketch}
\SectionName{ubshortLipschitz}

In this section we provide a proof sketch of the upper bound of the final iterate of SGD, in the case where $f$ is 1-Lipschitz but not necessarily strongly-convex (\Theorem{FinalIterateHighProbabilityLipschitz}). The proof of \Theorem{FinalIterateHighProbabilityLipschitz} closely resembles the proof of \Theorem{FinalIterateHighProbability} and we will highlight the main important differences. Perhaps the most notable difference is that the analysis in the Lipschitz case does not require a high probability bound on $\norm{x_t - x^*}^2$.

Recall that the step size used by \Algorithm{SGD} in this case is $\eta_t = 1/\sqrt{t}$.
We will write $\hat{g}_t = g_t - \hat{z}_t$, where $\hat{g}_t$ is the vector returned by the oracle at the point $x_t$, $g_t \in \partial f(x_t)$, and $\hat{z}_t$ is the noise.
Let $\cF_t = \sigma(\hat{z}_1, \ldots, \hat{z}_t)$ be the $\sigma$-algebra generated by the first $t$ steps of SGD. Finally, recall that $\norm{\hat{z}_t} \leq 1$ and $\expectg{\hat{z}_t}{\cF_{t-1}} = 0$.

As before, we begin with a lemma which can be obtained by modifying the proof of \Lemma{ShamirZhangHalf} to replace applications of strong convexity with the subgradient inequality.
\begin{lemma}
\LemmaName{ShamirZhangHalfLipschitz}
Let $f$ be 1-Lipschitz. Suppose that we run SGD (\Algorithm{SGD}) with step sizes $\eta_t = \frac{1}{\sqrt{t}}$. Then,
$$
f(x_T) ~\leq~ \underbrace{\frac{1}{T/2 + 1} \sum_{t=T/2}^T f(x_t)}_{\text{suffix average}} ~+~ \underbrace{\sum_{k=1}^{T/2}\frac{1}{k(k+1)}\sum_{t=T-k}^T \inner{\hat{z}_t}{x_t - x_{T-k}}}_{Z_T, \text{ the noise term}} ~+~ O\left( \frac{\log(T)}{\sqrt{T}}\right).
$$
\end{lemma}

\Lemma{ShamirZhangHalfLipschitz} asserts that the error of the last iterate is bounded by the sum of the error of the average of the iterates and some noise terms (up to the additive $ O  ( \log T / \sqrt{T}  )$ term). A standard analysis (similar to the proof of \Lemma{StandardSGDAnalysis}) reveals $\sum_{t=T/2}^T \left[  f(x_t) - f(x^*) \right] \leq O ( {\sqrt{T}} ) +   \sum_{t=T/2}^T \inner{\hat{z}_t}{x_t - x^*}$. Applying Azuma's inequality on the summation (using the diameter bound to obtain $\inner{\hat{z}_t}{x_t - x^*}^2 \leq 1$) shows 

\begin{lemma}
\LemmaName{AverageIterateHighProbLipschitz}
For every $\delta \in (0,1)$,
$$\frac{1}{T/2+1} \left[ \sum_{t=T/2}^T f(x_t) - f(x^*) \right] = O  \left( \sqrt{{\log(1/\delta)}/{T}}  \right),$$
with probability at least $1 - \delta$.
\end{lemma}

As a consequence of \Lemma{AverageIterateHighProbLipschitz}, it is enough to prove that the error due to the noise terms are small in order to complete the proof of \Theorem{FinalIterateHighProbabilityLipschitz}. By changing the order of summation, we can write  $Z_T = \sum_{t=T/2}^T \inner{\hat{z}_t}{w_t}$ where 
$$
w_t ~=~ \sum_{j=1}^{t-1} \alpha_j (x_t - x_j) \qquad\text{and}\qquad \alpha_j ~=~ \frac{1}{(T-j) (T-j+1)}. 
$$
Just as in \Section{ubshort}, the main technical difficulty is to show that $Z_T$ is small with high probability. Formally, we prove the following lemma, whose proof is outlined in \Subsection{BoundingTheNoiseLipschitz}. 
\begin{lemma}
\LemmaName{ZTBoundLipschitz}
 For every $\delta \in (0,1)$, $Z_T \leq O \left( {\log(T) \log(1/\delta)}/{\sqrt{T}} \right )$ with probability at least $1- \delta$. 
\end{lemma}
\noindent Given \Lemma{AverageIterateHighProbLipschitz} and \Lemma{ZTBoundLipschitz}, the proof of \Theorem{FinalIterateHighProbabilityLipschitz} is straightforward. The next sub-section provides a proof sketch of \Lemma{ZTBoundLipschitz}.


\subsection{Bounding the noise}
\SubsectionName{BoundingTheNoiseLipschitz}

The goal of this section is to prove \Lemma{ZTBoundLipschitz}. Just as in \Section{ubshort}, the main technical difficulty is to understand the noise term, denoted $Z_T$. Observe that $Z_T$ is a sum of a martingale difference sequence, and $\sum_{t=T/2}^T \norm{w_t}^2$ is the TCV of $Z_T$. The TCV of $Z_T$ will be shown to exhibit the ``chicken and egg'' relationship which we have already seen explicitly exhibited by the TCV of the noise terms in the strongly convex case. That is, we will see that the $\sum_{t=T/2}^T \norm{w_t}^2$ is bounded by a linear transformation of $Z_T$. We will again use our Generalized Freedman to disentangle the total conditional variance from the martingale.

The distance $\norm{x_t - x_j}$ between SGD iterates is again a crucial quantity to understand in order to bound $\sum_{t=T/2}^{T} \norm{w_t}^2$ (see \Subsection{BoundingTheNoise} to see why). Therefore, we develop a distance estimate analogous to \Lemma{StochasticDistanceEstimate} 

\newcommand{\distanceestimateLipschitz}{
Suppose $f$ is 1-Lipschitz. Suppose we run \Algorithm{SGD} for $T$ iterations with step sizes $\eta_t = 1/\sqrt{t}$. Let $a < b$. Then,
\[
\norm{x_a - x_b }^2 \leq \sum_{i=a}^{b-1} \frac{\norm{\hat{g}_i}^2}{i} + 2 \sum_{i=a}^{b-1} \frac{\big (f(x_a) - f(x_i) \big)}{\sqrt{i}} + 2\sum_{i=a}^{b-1} \frac{\inner{\hat{z}_i}{x_i - x_a}}{\sqrt{i}}.
\]
}

\begin{lemma}
\LemmaName{StochasticDistanceEstimateLipschitz}
\distanceestimateLipschitz
\end{lemma}

We then use \Lemma{StochasticDistanceEstimateLipschitz} to
prove \Lemma{MainTechnicalLemmaLipschitz}, our main upper bound on
$\sum_{t=T/2}^T \norm{w_t}^2$.
This follows from some delicate calculations similar to those in \Appendix{w_t_upper_bound}, replacing the strongly-convex distance estimate (\Lemma{StochasticDistanceEstimate}) with the Lipschitz distance estimate (\Lemma{StochasticDistanceEstimateLipschitz}), along with some other minor modifications.
This upper bound reveals the surprisingly intricate relationship between $Z_T$ (the martingale) and $\sum_{t=T/2}^T \norm{w_t}^2$ (its TCV).  

\begin{lemma}[Main Technical Lemma (Lipschitz Case)] 
\LemmaName{MainTechnicalLemmaLipschitz}
There exists positive values $R_1 = O \left ( \frac{\log^2 T}{T}\right )$, $R_2 = O \left( \frac{\log T}{T^{1.5}}  \right)$, and $C_t = O \left ( \log T \right )$, such that
$$\sum_{t=T/2}^T \norm{w_t}^2 \leq R_1 + R_2 \sum_{t=T/2}^T \inner{\hat{z}_t}{x_t - x^*} + \underbrace{\sum_{t=T/2}^T \inner{\hat{z}_t}{\frac{C_t}{\sqrt{t}} w_t}}_{\approx O(\log T/\sqrt{T}) Z_T }.$$
\end{lemma}

Just as in \Lemma{w_t_upper_bound}, the left-hand side is an upper bound on the total conditional variance of $Z_T$, whereas the right-hand side essentially contains a scaled version of $Z_T$ itself.
This is another instance of the ``chicken and egg phenomenon'' alluded to in \Section{Techniques},
and it is the main challenge of bounding $Z_T$. For convenience, we restate our main probabilistic tool which allows us to deal with the chicken and egg phenomenon.

\repeatclaimwithname{\Theorem{FanV2}}{Generalized Freedman}{\fantwo}

\vspace{-6pt}

In order to apply \Theorem{FanV2}, we need to refine \Lemma{MainTechnicalLemmaLipschitz} to replace $R_2\sum_{t=T/2}^T \inner{\hat{z}_t}{x_t-x^*}$ with a sufficient high probability upper bound. This is similar to the refinement of \Lemma{w_t_upper_bound} from \Subsection{BoundingTheNoise}. However, unlike the refinement in \Subsection{BoundingTheNoise} (which required a high probability bound on $\sum_{t=T/2}^T A_t \norm{x_t - x^*}^2$ without any diameter bound), the refinement here is quite easy. Using the diameter bound, the almost sure bound of $\norm{\hat{z}_t} \leq 1$, and Azuma's inequality, we can bound $\sum_{t=T/2}^T \inner{\hat{z}_t}{x_t - x^*}$ by $\sqrt{T  \log(1/\delta)}$ with probability at least $1- \delta$. This yields the following lemma.

\begin{lemma}
\LemmaName{RevisedMainTechnicalLemmaLipschitz} For every $\delta \in (0,1)$, there exists positive values $R = O \left ( \frac{\log^2 T \sqrt{\log(1/\delta)} } {T} \right )$, $C_t = O \left( \log T \right )$, such that $\sum_{t=T/2}^T \norm{w_t}^2 \leq R + \sum_{t=T/2}^{T-1} \inner{\hat{z}_t}{\frac{C_t}{\sqrt{t}}  w_t },$ with probability at least $1- \delta$. 
\end{lemma}

Now that we have derived an upper bound on the total conditional variance of $Z_T$ in the form required by Generalized Freedman Inequality (\Theorem{FanV2}), we are now finally ready to prove \Lemma{ZTBoundLipschitz} (the high probability upper bound on the noise, $Z_T$).

\begin{proofof}{\Lemma{ZTBoundLipschitz}}
We have demonstrated that $Z_T$ satisfies the ``Chicken and Egg'' phenomenon with high probability. Translating this into a high probability upper bound on the martingale $Z_T$ itself is a corollary of \Theorem{FanV2}. 

Indeed, consider a filtration $\{\cF_t \}_{t=T/2}^T$. Let $d_t = \inner{a_t}{b_t}$ define a martingale difference sequence where $\norm{a_t}\leq 1$ and $\expectg{a_t}{\cF_{t-1}} = 0$. Suppose there are positive values, $R$, $\alpha_t$, such that $\max_{t=T/2}^T \{ \alpha_t \} = O\left(\sqrt{R} \right)$ and $\sum_{t=T/2}^T \norm{b_t}^2 \leq \sum_{t=T/2}^T \alpha_t d_t + R\sqrt{\log(1/\delta)}$ with probability at least $1-\delta$. Then, \Corollary{RelateChickenAndEggToMartingale} bounds the martingale at time step $T$ by $\sqrt{R} \log(1/\delta)$ with high probability. 

Observe that \Lemma{RevisedMainTechnicalLemmaLipschitz} allows us to apply \Corollary{RelateChickenAndEggToMartingale} with $a_t = \hat{z}_t$, $b_t = w_t$, $\alpha_t = (C_t/\sqrt{t})$ for $t = T/2, \ldots, T-1$, $\alpha_T = 0$, $\max_{t=T/2}^T \{ \alpha_t\} = O\left (\log T/\sqrt{T}\right )$, and $R = O\left ( \log^2 T / T \right )$ to prove \Lemma{ZTBoundLipschitz}.

\end{proofof}

\fi

%% file: appendix_standard.tex
\section{Standard results}
\AppendixName{standard}

\begin{lemma}[Exponentiated Markov]
\LemmaName{ExponentiatedMarkov}
Let $X$ be a random variable and $\lambda > 0$.
Then $\prob{X > t} \leq  \exp(-\lambda t)\expect{\exp(\lambda X)}$.
\end{lemma}
\begin{theorem}[Cauchy-Schwarz]
Let $X$ and $Y$ be random variables.
Then $\abs{\expect{XY}}^2 \leq \expect{X^2}\expect{Y^2}$.
\end{theorem}
\begin{theorem}[H\"{o}lder's Inequality]
\TheoremName{Holder}
Let $X_1, \ldots, X_n$ be random variables and $p_1, \ldots, p_n > 0$ be such that
$\sum_i 1/p_i = 1$.
Then $\expect{\prod_{i=1}^n \abs{X_i}} \leq \prod_{i=1}^n \left( \expect{\abs{X_i}^p} \right)^{1/p_i}$
\end{theorem}
\begin{lemma}
\LemmaName{MGFHolder}
Let $X_1, \ldots, X_n$ be random variables and $K_1, \ldots, K_n > 0$ be such that $\expect{\exp(\lambda X_i)} \leq \exp(\lambda K_i)$ for all $\lambda \leq 1 / K_i$.
Then $\expect{\exp(\lambda \sum_{i=1}^n X_i)} \leq \exp(\lambda \sum_{i=1}^n K_i)$ for all $\lambda \leq 1 / \sum_{i=1}^n K_i$.
\end{lemma}
\begin{proof}
Let $p_i = \sum_{j=1}^n K_j / K_i$ and observe that $p_i K_i = \sum_{j=1}^n K_j$.
By assumption, if $\lambda p_i \leq 1/K_i$ (i.e.~$\lambda \leq 1/\sum_{j=1}^n K_j$) then $\expect{\exp(\lambda p_i X_i)} \leq \exp(\lambda p_i K_i)$.
Applying \Theorem{Holder}, we conclude that
\[
\expect{\exp(\lambda \sum_{i=1}^n X_i)} \leq \prod_{i=1}^n \expect{\exp(\lambda p_i X_i)}^{1/p_i} \leq \prod_{i=1}^n \exp(\lambda p_i K_i)^{1/p_i} = \exp(\lambda \sum_{i=1}^n K_i). \qedhere
\]
\end{proof}

\begin{lemma}[Hoeffding's Lemma]
\LemmaName{HoeffdingsLemma}
Let $X$ be any real valued random variable with expected value $\expect{X} = 0$ and such that $a \leq X \leq b$ almost surely. Then, for all $\lambda \in \bR$, $\expect{\exp\left ( \lambda X \right)} \leq \exp \left( \lambda^2(b-a)^2/8 \right)$. 
\end{lemma}

\begin{claim}[{\cite[Proposition~2.5.2]{Ver18}}]
\ClaimName{MeanZeroSubGaussian}
Suppose there is $c>0$ such that for all $ 0 < \lambda \leq \frac{1}{c}$, $\expect{\exp \big ( \lambda^2 X^2 \big ) } \leq   \exp \big ( \lambda^2 c^2 \big )$ for some constant $c$. Then, if $X$ is mean zero it holds that
$$
\expect{ \exp \big ( \lambda X \big)     } \leq   \exp \big  ( \lambda^2 c^2 \big ),
$$ for all $\lambda \in \bR$. 
\end{claim}
\begin{proof}
Without loss of generality, assume $c = 1$; otherwise replace $X$ with $X/c$.
Using the numeric inequality $e^x \leq x + e^{x^2}$ which is valid for all $x \in \bR$, if $\abs{\lambda} \leq 1$ then $\expect{\exp(\lambda X)} \leq \expect{\lambda X} + \expect{\exp(\lambda^2 X^2)} \leq \exp(\lambda^2)$.
On the other hand, if $\abs{\lambda} \geq 1$, we may use the numeric inequality\footnote{Young's Inequality} $ab \leq a^2/2 + b^2/2$, valid for all $a, b \in \bR$, to obtain
\[
\expect{\exp(\lambda X)} \leq \expect{\exp(\lambda^2/2 + X^2/2)} \leq \exp(\lambda^2/2) \exp(\lambda^2 / 2) = \exp(\lambda^2).
\qedhere
\]
\end{proof}

\begin{claim}
\ClaimName{MGFToTailBound}
Suppose $X$ is a random variable such that there exists constants $c$ and $C$ such that $\expect{\exp \left ( \lambda X \right )} \leq  c \exp \left ( \lambda C  \right )$ for all $\lambda \leq 1/C$. Then, $\prob{X \geq C \log(1/\delta) } \leq c \me \delta$. 
\end{claim}

\begin{proof}
Apply \Lemma{ExponentiatedMarkov} to $\prob{ X \geq t}$ to get $\prob{X \geq t} \leq c\exp\left( -\lambda t + \lambda C\right)$. Set $\lambda = 1/C$ and $t = C\log(1/\delta)$ to complete the proof.
\end{proof}

\begin{claim}[{\cite[Eq.~(3.1.6)]{HUL}}]
\ClaimName{ProjContraction}
Let $\cX$ be a convex set and $x \in \cX \subseteq \bR^n$.
Then $\norm{\Pi_{\cX}(y) - x} \leq \norm{y - x}$ for all $y \in \bR^n$.
\end{claim}

\subsection{Useful Scalar Inequalities}
\begin{claim}
\ClaimName{SumRecipSqrt}
For $1 \leq a \leq b$,
$\sum_{k=a}^b \frac{1}{\sqrt{k}} \leq 2 \frac{b-a+1}{\sqrt{b}}$.
\end{claim}
\begin{proof}
$$
\sum_{k=a}^b \frac{1}{\sqrt{k}}
    ~\leq~ \int_{a-1}^b \frac{1}{\sqrt{x}} \, dx 
    ~=~ 2 (\sqrt{b} - \sqrt{a-1})
    ~=~ 2 \frac{b-a+1}{\sqrt{b} + \sqrt{a-1}}.
$$
\end{proof}

\begin{claim}
\ClaimName{WeirdInequality}
For any $1 \leq j \leq t \leq T$,
we have $ \frac{t-j}{(T-j+1) \sqrt{t}} \leq \frac{1}{\sqrt{T}} $.
\end{claim}
\begin{proof}
The function $g(x) = \frac{x-j}{\sqrt{x}}$
has derivative
$$g'(x)
= \frac{1}{\sqrt{x}} \Big(1-\frac{x-j}{2x}\Big)
= \frac{1}{\sqrt{x}} \Big(\frac{1}{2}+\frac{j}{2x}\Big)
.$$
This is positive for all $x>0$ and $j \geq 0$, and so
$$
\frac{t-j}{\sqrt{t}} ~\leq~ \frac{T-j}{\sqrt{T}},
$$
for all $0 < t \leq T$. This implies the claim.
\end{proof}

\begin{claim}
\ClaimName{PartialSumOfRecip}
$$\sum_{\ell=k+1}^{m} \frac{1}{\ell^2} ~\leq~ \frac{1}{k} - \frac{1}{m}.$$
\end{claim}
\begin{proof}
The sum may be upper-bounded by an integral as follows:
$$\sum_{\ell=k+1}^{m} \frac{1}{\ell^2}
~\leq~ \int_k^m \frac{1}{x^2} \,dx
~=~ \frac{1}{k} - \frac{1}{m}.
\qedhere
$$
\end{proof}

\begin{claim}
\ClaimName{SumOfAlphas}
Let $\alpha_j = \frac{1}{(T-j)(T-j+1)}$. Let $a, b$ be such that $a < b \leq T$. Then,
$$
\sum_{j = a}^b \alpha_j ~=~ \frac{1}{T-b} - \frac{1}{T-a+1} ~\leq~ \frac{1}{T-b}.
$$
\end{claim}
\begin{proof}
\begin{align*}
\sum_{j=a}^b \alpha_j ~=~ \sum_{j=a}^b \frac{1}{(T-j)(T-j+1)} ~=~ \sum_{j=a}^b \bigg ( \frac{1}{T-j} - \frac{1}{T- (j-1)} \bigg ) ,
\end{align*}which is a telescoping sum. 
\end{proof}

\begin{claim}
\ClaimName{logToFraction}
Suppose $a < b$. Then, $\log(b/a) \leq (b-a)/a$. 
\end{claim}

\begin{claim}
\ClaimName{HalfHarmonicSum}
Let $b \geq a > 1$. Then, $\sum_{i=a}^b \frac{1}{i} \leq \log \big ( b  / (a-1) \big )$. 
\end{claim}

%% file: appendix_lb_omitted.tex
\iffull
	\section{Omitted proofs for the lower bounds}
    
    \subsection{Strongly convex case}
	\AppendixName{LowerBound}
\fi

\begin{proofof}{\Lemma{XisZSC}}
By definition, $z_1 = x_1 = 0$.
By \Claim{OracleSubgradientSC}, the subgradient returned at $x_1$ is $h_1+x_1=h_1$, 
so \Algorithm{SGD} sets $y_2 = x_1-\eta_1 h_1 = e_1$, the first standard basis vector.
Then \Algorithm{SGD} projects onto the feasible region,
obtaining $x_2 = \Pi_\cX(y_2)$, which equals $e_1$ since $y_2 \in \cX$.
Since $z_2$ also equals $e_1$, the base case is proven.

So assume $z_t = x_t$ for $2 \leq t < T$; we will prove that $z_{t+1} = x_{t+1}$.
By \Claim{OracleSubgradientSC}, the subgradient returned at $x_t$ is $\hat{g}_t = h_t+z_t$.
Then \Algorithm{SGD} sets $y_{t+1} = x_t - \eta_t \hat{g}_t$.
Since $x_t = z_t$ and $\eta_t = 1/t$, we obtain
\begin{align*}
y_{t+1,j}
    &~=~ z_{t,j} - \frac{1}{t} (h_{t,j}+z_{t,j}) \\
    &~=~ \frac{t-1}{t} z_{t,j} - \frac{1}{t} h_{t,j} \\
    &~=~ \frac{t-1}{t} 
        \left\lbrace
        \begin{array}{ll}
        \frac{1-(t-j-1)a_j}{t-1} &~~\text{(for $j<t$)} \\
        0                &~~\text{(for $j \geq t$)}
        \end{array}
        \right\rbrace
    ~-~ \frac{1}{t}
        \left\lbrace
        \begin{array}{ll}
        a_j     &~~\text{(for $j<t$)} \\
        -1      &~~\text{(for $j=t$)} \\
        0       &~~\text{(for $j>t$)}
        \end{array}
        \right\rbrace
        \\
    &~=~ \frac{1}{t} \left\lbrace
        \begin{array}{ll}
        1-(t-j-1)a_j     &~~\text{(for $j<t$)} \\
        0                &~~\text{(for $j \geq t$)}
        \end{array}
        \right\rbrace
    ~-~ \frac{1}{t}
        \left\lbrace
        \begin{array}{ll}
        a_j    &~~\text{(for $j<t$)} \\
        -1       &~~\text{(for $j=t$)} \\
        0       &~~\text{(for $j>t$)}
        \end{array}
        \right\rbrace
        \\
    &~=~ \frac{1}{t}
        \left\lbrace
        \begin{array}{ll}
        1-(t-j)a_j  &\quad\text{(for $j<t$)} \\
        1           &\quad\text{(for $j=t$)} \\
        0           &\quad\text{(for $j \geq t+1$)}
        \end{array}
        \right\rbrace
\end{align*}
So $y_{t+1} = z_{t+1}$.
Since $x_{t+1} = \Pi_\cX(y_{t+1})$ is defined to be the projection onto $\cX$,
and $y_{t+1} \in \cX$ by \Claim{CombinedClaimAboutZ},
we have $x_{t+1} = y_{t+1} = z_{t+1}$.
\end{proofof}

\begin{claim}
\ClaimName{LBForSuffix}
For any $k \in [T]$, let $\bar{x} = \sum_{t=T-k+2}^{T+1} \lambda_t x_t$
be any convex combination of the last $k$ iterates.
Then
$$f(\bar{x}) ~\geq~ \frac{\ln(T) - \ln(k)}{4 T}.$$
\end{claim}
\begin{proof}
By \Lemma{XisZSC}, $x_t = z_t ~\forall t \in [T+1]$.
By \Claim{CombinedClaimAboutZ}, every $z_t \geq 0$ so $\bar{x} \geq 0$.
Moreover, $z_{t,j} \geq 1/2T$ for all $T-k+2 \leq t \leq T+1$ and $1 \leq j \leq T-k+1$.
Consequently, $\bar{x}_j \geq 1/2T$ for all $1 \leq j \leq T-k+1$.
Thus,
\begin{align*}
f(\bar{x})
    &~\geq~ h_{T+1} \transpose \bar{x}
        \qquad\text{(by definition of $f$)}\\
    &~=~ \sum_{j=1}^{T-k+1} h_{T+1,j} \underbrace{\bar{x}_j}_{\geq 1/2T}
        ~+~ \sum_{j=T-k+2}^{T} \underbrace{h_{T+1,j} \, \bar{x}_j}_{\geq 0} \\
    &~\geq~ \sum_{j=1}^{T-k+1} a_j \cdot \frac{1}{2T} \\
    &~=~ \frac{1}{2T} \sum_{j=1}^{T-k+1} \frac{1}{2(T+1-j)} \\
    &~\geq~ \frac{1}{4T} \sum_{j=1}^{T-k+1} \frac{1}{T+1-j} \\
    &~\geq~ \frac{1}{4T} \int_{1}^{T-k+1} \frac{1}{T+1-x} \,dx \\
    &~=~ \frac{\log(T) - \log(k)}{4 T}  \qedhere
\end{align*}
\end{proof}

\iffull
	\subsection{Lipschitz case}
    \AppendixName{LowerBoundLipschitz}

	\begin{claim}
    \ClaimName{LBForSuffixLipschitz}
For any $k \in [T]$, let $\bar{x} = \sum_{i=T-k+1}^T \lambda_i x_i$
be any convex combination of the last $k$ iterates.
Then
$$f(\bar{x}) ~\geq~ \frac{\ln(T) - \ln(k+1)}{16 \sqrt{T}}.$$
\end{claim}
\begin{proof}
By \Lemma{XisZ}, $x_i = z_i$ for all $i$.
By \Claim{ZNonNegative}, every $z_i \geq 0$ so $\bar{x} \geq 0$.
Moreover, $z_{i,j} \geq 1/2\sqrt{T}$ for all $T-k+1 \leq i \leq T$ and $1 \leq j \leq T-k$,
and $z_{i,T}=0$ for all $i \leq T$.
Consequently, $\bar{x}_j \geq 1/2\sqrt{T}$ for all $1 \leq j \leq T-k$ and $\bar{x}_T=0$.
Thus,
\begin{align*}
f(\bar{x})
    &~\geq~ h_T \transpose \bar{x}
        \quad\text{(by definition of $f$)}\\
    &~=~ \sum_{j=1}^{T-k} h_{T,j} \bar{x}_j 
        ~+~ \sum_{j=T-k+1}^{T-1} \underbrace{h_{T,j} \bar{x}_j}_{\geq 0}
        ~+~ h_{T,T} \underbrace{\bar{x}_T}_{=0} \\
    &~\geq~ \sum_{j=1}^{T-k} a_j \frac{1}{2 \sqrt{T}} \\
    &~=~ \frac{1}{2 \sqrt{T}} \sum_{j=1}^{T-k} \frac{1}{8(T-j+1)} \\
    &~\geq~ \frac{1}{16 \sqrt{T}} \int_{1}^{T-k} \frac{1}{T-x+1} \,dx \\
    &~=~ \frac{\log(T) - \log(k+1)}{16 \sqrt{T}} 
\end{align*}
\end{proof}

\subsection{Monotonicity}
\SubsectionName{LBMonotone}

The following claim completes the proof of \eqref{eq:Monotone},
under the assumption that $\eta_t = \frac{c}{\sqrt{t}}$.

\begin{claim}
For any $i \leq T$, we have $f(x_{i+1}) \geq f(x_i) + 1/32c \sqrt{T} (T-i+1)$.
\end{claim}
\begin{proof}
\begin{align*}
f(x_{i+1}) - f(x_i) 
    &~=~ h_{i+1} \transpose z_{i+1} - h_i \transpose z_i
        \qquad\text{(by \Claim{OracleSubgradient})}\\
    &~=~ \sum_{j=1}^i (h_{i+1,j} z_{i+1,j} - h_{i,j} z_{i,j} ) \\
    &~=~ \sum_{j=1}^{i-1} (h_{i+1,j} z_{i+1,j} - h_{i,j} z_{i,j})
        + (h_{i+1,i} z_{i+1,i} - h_{i,i} \underbrace{z_{i,i}}_{=0} ) \\
    &~=~ \sum_{j=1}^{i-1} a_j (z_{i+1,j} - z_{i,j}) + a_i z_{i+1,i} \\
    &~=~ \sum_{j=1}^{i-1} a_j \cdot \Big(\frac{-c a_j}{\sqrt{i}}\Big)
        \:+\: \frac{1}{8c(T-i+1)} z_{i+1,i} \\
    &~\geq~ - \frac{1}{64c \sqrt{i}} \sum_{j=1}^{i-1} \Big(\frac{1}{T-j+1}\Big)^2 
        \:+\: \frac{1}{32c \sqrt{T}(T-i+1)}
        \qquad\text{(by \Claim{ZNonNegative})} \\
    &~\geq~ \frac{1}{32c \sqrt{T}(T-i+1)}
        \qquad\text{(by \Claim{WeirdInequality2})} 
\end{align*}
\end{proof}

\begin{claim}
\ClaimName{WeirdInequality2}
$$
\frac{1}{\sqrt{i}} \sum_{j=1}^{i-1} \Big(\frac{1}{T-j+1}\Big)^2 
    ~\leq~ 
\frac{1}{ \sqrt{T}} \cdot \frac{1}{T-i+1}.
$$
\end{claim}
\begin{proof}
Applying \Claim{PartialSumOfRecip} shows that
$$
\sum_{j=1}^{i-1} \Big(\frac{1}{T-j+1}\Big)^2 
    ~=~ \sum_{\ell=T-i+2}^{T} \frac{1}{\ell^2}
    ~\leq~ \frac{1}{T-i+1} - \frac{1}{T}
    ~=~ \frac{i}{T(T-i+1)}.
$$
So it suffices to prove that
$$
\frac{\sqrt{i}}{T(T-i+1)}
    ~\leq~ \frac{1}{\sqrt{T}} \cdot \frac{1}{T-i+1}.
$$
This obviously holds as $i \leq T$.
\end{proof}

	\subsection{A function independent of $T$}
\fi

\begin{remark}
\RemarkName{InfiniteDimension}
In order to achieve large error for the $T\th$ iterate, \Theorem{FinalIterateLowerBound}
constructs a function parameterized by $T$.
It is not possible for a \emph{single} function to achieve error $\omega(1/T)$ for the $T\th$ iterate
simultaneously for \emph{all} $T$, because that would contradict the fact that suffix averaging achieves error $O(1/T)$.
Nevertheless, it is possible to construct a single function achieving error $g(T)$, for infinitely many $T$,
for any function $g(T) = o(\log T / T)$, e.g., $g(T) = \log T / (T \log^*(T))$ where $\log^*(T)$ is the iterated logarithm.
Formally, we can construct a function $f \in \ell_2$ such that $\inf_{x} f(x) = 0$ but
$\limsup_{T} \frac{f(x_T)}{g(T)} = +\infty$.
The main idea is to define a sequence $T_1 \ll T_2 \ll T_3 \ll \ldots$ and consider the ``concatenation'' of $c_1f_{T_1}, c_2f_{T_2}, \ldots$ into a single function $f$ (here, $c_i$ are appropriate constants chosen to ensure that $f$ remains Lipschitz).
Essentially, one can imagine running multiple instances of gradient descent in parallel where each instance corresponds to a bad instance given by \Theorem{FinalIterateLowerBound}, albeit at different scales.
However, this construction has a slight loss (i.e., the $\log^*(T)$) to ensure that $f$ remains Lipschitz.
The details are discussed in the full version of this paper.
\end{remark}

%% file: appendix_fanv2.tex
\section{Proof of \Theorem{FanV2} and Corollaries}
\AppendixName{FanV2}

In this section we prove \Theorem{FanV2} and derive some corollaries. We restate \Theorem{FanV2} here for convenience.

\repeatclaim{\Theorem{FanV2}}{\fantwo}
\begin{proofof}{\Theorem{FanV2}}
Fix $\lambda < 1/(2\alpha)$ and define $c = c(\lambda, \alpha)$ as in \Claim{FanRoots}.
Let $\tilde{\lambda} = \lambda + c \lambda^2 \alpha$.
Define $\cU_0 \coloneqq 1$ and for $t \in [n]$, define
\[
\cU_t(\lambda) \coloneqq \exp\left( \sum_{i=1}^t (\lambda + c \lambda^2 \alpha_i) d_i - \sum_{i=1}^t \frac{\tilde{\lambda}^2}{2} v_{i-1} \right).
\]
\begin{claim}
$\cU_t(\lambda)$ is a supermartingale w.r.t.~$\cF_t$.
\end{claim}
\begin{proof}
For all $t \in [n]$:
\begin{align*}
\expectg{\cU_{t}(\lambda)}{\cF_{t-1}}
& = \cU_{t-1}(\lambda) \exp\left( -\frac{\tilde{\lambda}^2}{2} v_{t-1} \right) \expectg{\exp\left((\lambda + c\lambda^2 \alpha_t)d_t\right)}{\cF_{t-1}} \\
& \leq \cU_{t-1}(\lambda) \exp\left( -\frac{\tilde{\lambda}^2}{2} v_{t-1} \right) \exp\left( \frac{(\lambda + c\lambda^2 \alpha_i)^2}{2} v_{t-1} \right)
\\
& \leq \cU_{t-1}(\lambda) \exp\left( -\frac{\tilde{\lambda}^2}{2} v_{t-1} \right) \exp\left( \frac{\tilde{\lambda}^2}{2} v_{t-1} \right)
\\
& = \cU_{t-1}(\lambda),
\end{align*}
where the second line follows from the assumption that $\expectg{\exp(\lambda d_t)}{\cF_{t-1}} \leq \exp\left( \frac{\lambda^2}{2} v_{t-1} \right)$ for all $\lambda > 0$ and
the third line is because $\lambda + c \lambda^2 \alpha_i \leq \tilde{\lambda}$ (since $c \geq 0$ and $\alpha_i \leq \alpha$).
We conclude that $\cU_t(\lambda)$ is a martingale w.r.t.~$\cF_{t}$.
\end{proof}
Define the stopping time $T = \min\setst{t}{S_t \geq x \text{ and } V_t \leq \sum_{i=1}^t \alpha_i d_i + \beta}$ with the convention that $\min \emptyset = \infty$.
Since $\cU_t$ is a supermartingale w.r.t.~$\cF_t$, $\cU_{T \wedge t}$ is a supermartingale w.r.t.~$\cF_t$.
Hence,
\begin{align*}
\prob{\bigcup_{t=1}^n \left\{S_t \geq x \text{ and } V_t \leq \sum_{i=1}^t \alpha_i d_i + \beta \right\} }
& =
\prob{S_{T \wedge n} \geq x \text{ and } V_{T \wedge n} \leq \sum_{i=1}^{T \wedge n} \alpha_i d_i + \beta }
\\
& =
\prob{\lambda S_{T \wedge n} \geq \lambda x \text{ and } c \lambda ^2 V_{T \wedge n} \leq c \lambda^2 \sum_{i=1}^{T \wedge n} \alpha_i d_i + c \lambda^2 \beta } \\
& \leq \prob{ \sum_{i=1}^{T \wedge n} (\lambda + \alpha_i \lambda^2) d_i - c\lambda^2 V_{T \wedge n} \geq \lambda x - c \lambda^2 \beta } \\
& \leq \expect{\exp\left( \sum_{i=1}^{T \wedge n} (\lambda + \alpha_i \lambda^2) d_i - c\lambda^2 V_{T \wedge n} \right)} \cdot \exp(-\lambda x + c \lambda^2 \beta).
\end{align*}
Recall that $c$ was chosen (via \Claim{FanRoots}) so that $c\lambda^2 = \frac{\tilde{\lambda}^2}{2}$.
Hence,
\begin{align*}
\expect{\exp\left( \sum_{i=1}^{T \wedge n} (\lambda + \alpha_i \lambda^2) d_i - c\lambda^2 V_{T \wedge n} \right)} &
=
\expect{\exp\left( \sum_{i=1}^{T \wedge n} (\lambda + \alpha_i \lambda^2) d_i - \frac{\tilde{\lambda}^2}{2} V_{T \wedge n} \right)} \\
& = \expect{\cU_{T \wedge n}(\lambda)} \leq 1.
\end{align*}
Since $\lambda < 1/(2\alpha)$ was arbitrary, we conclude that
\begin{align*}
\prob{\bigcup_{t=1}^n \left\{S_t \geq x \text{ and } V_t \leq \sum_{i=1}^t \alpha_i d_i + \beta \right\} }
& \leq \exp(-\lambda x + c \lambda^2 \beta) \\
& \leq \exp(-\lambda x + 2 \lambda^2 \beta),
\end{align*}
where the inequality is because $c \leq 2$.
Now, we can pick $\lambda = \frac{1}{2\alpha + 4\beta / x} < \frac{1}{2\alpha}$ to conclude that
\begin{align*}
\prob{\bigcup_{t=1}^n \left\{S_t \geq x \text{ and } V_t \leq \sum_{i=1}^t \alpha_i d_i + \beta \right\} }
& \leq \exp(-\lambda (x - 2 \lambda \beta)) \\
& \leq \exp\left( - \lambda \left( x - \frac{2\beta}{2\alpha + 4\beta/x} \right) \right) \\
& \leq \exp\left( - \lambda \left( x - \frac{2\beta}{4\beta/x} \right) \right) \\
& = \exp\left( - \frac{\lambda x}{2} \right) \\
& = \exp\left( - \frac{x}{4 \alpha + 8 \beta / x} \right). \qedhere
\end{align*}
\end{proofof}

\begin{claim}
\ClaimName{FanRoots}
Let $\alpha \geq 0$ and $\lambda \in [0, 1/2\alpha)$.
Then there exists $c = c(\lambda, \alpha) \in [0,2]$ such that $2c \lambda^2 = (\lambda + c\lambda^2 \alpha)^2$.
\end{claim}
\begin{proof}
If $\lambda = 0$ or $\alpha = 0$ then the claim is trivial (just take $c = 0$). So assume $\alpha, \lambda > 0$.

The equality $2 c \lambda^2 = (\lambda + c\lambda^2 \alpha)^2$ holds if and only if $p(c) \coloneqq \alpha^2 \lambda^2 c^2 + (2\lambda \alpha - 2)c + 1 = 0$.
The discriminant of $p$ is $(2\lambda \alpha - 2)^2 - 4\alpha^2 \lambda^2 = 4 - 8\lambda \alpha$.
Since $\lambda \alpha \leq 1/2$, the discriminant of $p$ is non-negative so the roots of $p$ are real.
The smallest root of $p$ is located at
\begin{align*}
c & = \frac{2 - 2\alpha \lambda - \sqrt{(2 \alpha \lambda - 2)^2 - 4 \lambda^2 \alpha^2}}{2\lambda^2 \alpha^2} \\
& = \frac{1 - \alpha \lambda - \sqrt{1 - 2 \alpha \lambda}}{\alpha^2 \lambda^2}.
\end{align*}
Set $\gamma = \alpha \lambda$.
Using the numeric inequality $\sqrt{1 - x} \geq 1 - x/2 - x^2/2$ valid for all $x \leq 1$, we have
\[
c \leq \frac{1 - \gamma - (1-\gamma - 2\gamma^2)}{\gamma^2} = 2.
\]
On the other hand, using the numeric inequality $\sqrt{1 - x} \leq 1 - x/2 - x^2/8$ valid for all $0 \leq x \leq 1$, we have
\[
c \geq \frac{1-\gamma - (1-\gamma - \gamma^2/2)}{\gamma^2} = \frac{1}{2} \geq 0. \qedhere
\]
\end{proof}

\subsection{Corollaries of \Theorem{FanV2}}

In this paper, we often deal with martingales, $M_n$, where the total conditional variance of the martingale is bounded by a linear transformation of the martingale, \emph{with high probability} (which is what we often refer to as the ``chicken and egg'' phenomenon --- the bound on the total conditional variance of $M_n$ involves $M_n$ itself). Transforming these entangled high probability bounds on the total conditional variance into high probability bounds on the martingale itself are easy consequences of our Generalized Freedman inequality (\Theorem{FanV2}).

\begin{lemma}
\LemmaName{HighProbChickenAndEggVarBound}
Let $\{ d_i, \cF_i \}_{i=1}^n$ be a martingale difference sequence. Let $v_{i-1}$ be a $\cF_{i-1}$ measurable random variable such that $\expectg{\exp \left ( \lambda d_i \right )}{\cF_{i-1}} \leq \exp \left (  \frac{\lambda^2}{2} v_{i-1}\right )$ for all $\lambda  >0$ and for all $i \in [n]$. Define $S_n = \sum_{i=1}^n d_i$ and define $V_n = \sum_{i=1}^n v_{i-1}$. Let $\delta \in (0,1)$ and suppose there are positive values $R(\delta)$, $\{ \alpha_i \}_{i=1}^n$ such that $\prob{  V_n \leq \sum_{i=1}^n \alpha_i d_i + R(\delta) } \geq 1-\delta$. Then, 
$$
\prob{S_n \geq x} \leq \delta + \exp \left ( - \frac{x^2}{4  \left (\max_{i=1}^n \{\alpha_i\} \right) x  + 8 R(\delta) } \right ).
$$
\end{lemma}

\begin{proof}
Fix $\delta \in (0,1)$. Define the following events: $\cE(x) = \{ S_n \geq x \}$, $\cG = \{ V_n \leq  \sum_{i=1}^n \alpha_i d_i + R(\delta)  \}$.

\begin{align*}
\prob{S_n \geq x} 
	&~=~ \prob{  \cE(x) \wedge \cG } ~+~ \prob{\cE(x) \wedge \cG^c } \\
    &~\leq~ \prob{  \cE(x) \wedge \cG } + \underbrace{\prob{ \cG^c }}_{\leq \delta} \\ 
    &~\leq~ \delta  + \exp\left( -\frac{x^2}{4 \left (\max_{i=1}^n \{a_i  \} \right )x + 8 R(\delta)   } \right),
\end{align*} where the final inequality is due to applying \Theorem{FanV2} to $\prob{  \cE(x) \wedge \cG }$.
\end{proof}

\noindent In this paper, we use \Lemma{HighProbChickenAndEggVarBound} in the following ways: 
\begin{corollary}
\CorollaryName{RelateVarianceToMartingale}
Let $\{ \cF_t \}_{t=1}^T$ be a filtration and suppose that $a_t$ are $\cF_{t}$-measurable random variables and $b_t$ are $\cF_{t-1}$-measurable random variables.
Further, suppose that
\begin{enumerate}[nosep]
\item
$\norm{a_t} \leq 1$ almost surely and $\expectg{a_t}{\cF_{t-1}} = 0$; and
\item
$\sum_{t=1}^T \norm{b_t}^2 \leq R \log(1/\delta)$ with probability at least $1 - O(\delta)$.
\end{enumerate}
Define $d_t = \inner{a_t}{b_t}$.
Then $\sum_{t=1}^T d_t ~\leq~ O\big(\sqrt{R}\log(1/\delta)\big)$ with probability at least $1 - O (\delta)$. 
\end{corollary}

\begin{proof}
Since $\norm{a_t} \leq 1$, by Cauchy-Schwarz we have that $\Abs{d_t} \leq \norm{b_t}$. Therefore, $\expectg{\exp \big ( 
\lambda d_t \big)}{\cF_{t-1}} \leq \exp \big( \frac{\lambda^2}{2} \norm{b_t}^2 \big )$ for all $\lambda$ by \Lemma{HoeffdingsLemma}. 
Next, applying \Lemma{HighProbChickenAndEggVarBound} with $d_t = \inner{a_t}{b_t}$ and $v_{t-1}  = \norm{b_t}^2$, $\alpha_i = 0$ for all $i$, and $R(\delta) = R \log(1/\delta)$ yields
$$
\prob{ \sum_{t=1}^T d_t  \geq x} ~\leq~ \delta + \exp \left ( - \frac{x^2}{8 R \log(1/\delta)} \right ).
$$ The last term is at most $\delta$ by taking $x = \sqrt{8 R} \log(1/\delta)$.
\end{proof}

\iffull

\begin{corollary}
\CorollaryName{RelateChickenAndEggToMartingale}
Let $\{ \cF_t \}_{t=1}^T$ be a filtration and suppose that $a_t$ are $\cF_{t}$-measurable random variables and $b_t$ are $\cF_{t-1}$-measurable random variables. Define $d_t = \inner{a_t}{b_t}$.
Assume that $\norm{a_t} \leq 1$ almost surely and $\expectg{a_t}{\cF_{t-1}} = 0$. Furthermore, suppose that there exists positive values $R$ and $\{ \alpha_t \}_{i=1}^{T-1}$ where $\max\{\alpha_t\}_{t=1}^{T-1} = O\left( \sqrt{R} \right)$, such that exactly one of the following holds for every $\delta \in (0,1)$ 
\begin{enumerate}[nosep]
\item 
$\sum_{t=1}^T \norm{b_t}^2 \leq  \sum_{t=1}^{T-1} \alpha_t d_t +  R \log(1/\delta)$ with probability at least $1 - O(\delta)$.
\item $\sum_{t=1}^T\norm{b_t}^2 \leq  \sum_{t=1}^{T-1} \alpha_t d_t +  R \sqrt{\log(1/\delta)}$ with probability at least $1 - O(\delta)$.
\end{enumerate}
\noindent Then $\sum_{t=1}^T d_t ~\leq~ O\big(\sqrt{R}\log(1/\delta)\big)$ with probability at least $1 - \delta $. 
\end{corollary}

\begin{proof}
We prove only the first case, the second case can be proved by bounding $\sqrt{\log(1/\delta)}$ by $\log(1/\delta)$ and using the proof of the first case.

Since $\norm{a_t} \leq 1$, by Cauchy-Schwarz we have that $\Abs{d_t} \leq \norm{b_t}$. Therefore, $\expectg{\exp \big ( 
\lambda d_t \big)}{\cF_{t-1}} \leq \exp \big( \frac{\lambda^2}{2} \norm{b_t}^2 \big )$ for all $\lambda$ by \Lemma{HoeffdingsLemma}. 
Next, applying \Lemma{HighProbChickenAndEggVarBound} with $d_t = \inner{a_t}{b_t}$ and $v_{t-1}  = \norm{b_t}^2$, with $\alpha_T = 0$ , and $R(\delta) = R \log(1/\delta)$ yields
$$
\prob{ \sum_{t=1}^T d_t  \geq x} ~\leq~ \delta + \exp \left ( - \frac{x^2}{4 \left(\max_{t=1}^{T-1} \{ \alpha_t \} \right)x + 8 R \log(1/\delta)} \right ).
$$ The last term is at most $\delta$ by taking $x = \Theta \left(\sqrt{ R} \log(1/\delta) \right)$ because $\max_{t=1}^{T-1} \{ \alpha_t\} = O \left(\sqrt{R}\right)$.

\end{proof}

\fi

%% file: appendix_recursive_mgf.tex
\section{Proof of \Theorem{RecursiveStochasticProcess}}
\AppendixName{RecursiveStochasticProcess}

\repeatclaim{\Theorem{RecursiveStochasticProcess}}{\recursiveprocess}

\begin{proofof}{\Theorem{RecursiveStochasticProcess}}

We begin by deriving a recursive MGF bound on $X_t$. 

\begin{claim}
\ClaimName{RecursiveMGFBound}
Suppose $0 \leq \lambda \leq \min_{1 \leq t \leq T} \left ( \frac{1-\alpha_t}{2 \beta_t^2} \right )$. Then for every $t$,
\[
\expect{\exp \left(\lambda X_{t+1} \right )} ~\leq~ \exp \left ( \lambda \gamma_t \right )\expect{ \exp \left (  \lambda X_t\left (\frac{1+\alpha_t}{2} \right ) \right ) }.
\]
\end{claim}

\begin{proof}

Observe that $ \beta_t^2 \hat{w}_t^2\sqrt{X_t}^2 \leq \beta_t^2 X_t$ because $\Abs{\hat{w}_t} \leq 1$ almost surely. Since $\beta_t^2 X_t$ is $\cF_{t}$-measurable, we have $\expectg{\exp\left( \lambda^2  \beta_t^2 \hat{w}_t^2\sqrt{X_t}^2 \right)}{\cF_{t}} \leq \exp \left ( \lambda^2 \beta_t^2 X_t \right )$ for all $\lambda$. Hence, we may apply \Claim{MeanZeroSubGaussian} to obtain 
\begin{equation}
\EquationName{IntermediateRecursiveMGFResult}
\expectg{  \exp \left ( \lambda \beta_t \hat{w}_t \sqrt{X_t}  \right ) }{\cF_{t}} ~\leq~ \exp \left ( \lambda^2 \beta_t^2 X_t \right ).
\end{equation}
Hence, 
\begin{align*}
\expect{ \exp \left ( \lambda X_{t+1} \right )  } 
	&~\leq~ \expect{ \exp \left ( \lambda \alpha_t X_t + \lambda \beta_t \hat{w}_t\sqrt{X_t} + \lambda \gamma_t  \right ) } \qquad\text{(by assumption)} \\
    &~=~ \expect{\exp \left ( \lambda \alpha_t X_t + \lambda \gamma_t \right) \expectg{\exp  \left (   \lambda \beta_t \hat{w}_t \sqrt{X_t}  \right )}{\cF_{t}}} \\
    &~\leq~ \expect{  \exp \left ( \lambda \alpha_t X_t  + \lambda^2 \beta_t^2 X_t + \lambda \gamma_t \right )  } \qquad\text{(by \Equation{IntermediateRecursiveMGFResult})} \\
   & ~=~ \expect{\exp \left ( \lambda X_t \left( \alpha_t + \lambda \beta_t^2 \right ) + \lambda  \gamma_t \right )} \\
   &~\leq~ \expect{ \exp \left ( \lambda \gamma_t + \lambda X_t\left (\frac{1+\alpha_t}{2} \right ) \right ) } \qquad \text{(because $\lambda \leq \frac{1 - \alpha_t}{2 \beta_t^2}$)}. \qedhere
\end{align*}
\end{proof}

\noindent Next, we prove an MGF bound on $X_t$.

\begin{claim}
\ClaimName{MGFBoundToyProblem} For every $t$ and for all $0 \leq \lambda \leq 1/K$, $\expect{\exp \left (\lambda X_t \right)} \leq \exp \left ( \lambda K \right )$.
\end{claim}

\begin{proof}
Let $\lambda \leq 1/K$. We proceed by induction over $t$. Assume that $\expect{\exp \left ( \lambda X_t \right )} \leq \exp \left ( \lambda K \right )$. Now, consider the MGF of $X_{t+1}$:
\begin{align*}
\expect{\exp \left (  \lambda X_{t+1}\right )} 
	&~\leq~  \expect{ \exp \left ( \lambda \gamma_t + \lambda X_t\left (\frac{1+\alpha_t}{2} \right ) \right ) } \qquad\text{(by \Claim{RecursiveMGFBound})} \\
    &~\leq~ \exp\left(\lambda \gamma_t + \lambda K  \left (\frac{1 + \alpha_t}{2}  \right )\right ), 
\end{align*}where the first inequality is valid because $\lambda \leq 1/K \leq \min_{1 \leq t \leq T} \left ( \frac{1 - \alpha_t}{2 \beta_t^2} \right)$ and the second inequality follows because $(1 + \alpha_t)/2 < 1$ and so we can use the induction hypothesis since $\lambda (1 + \alpha_t)/2 < \lambda \leq 1/K$. Furthermore, because $K \geq 2\gamma_t/\left ( 1 - \alpha_t \right)$ we have 
\[ K ~\geq~ \frac{2 \gamma_t}{1 - \alpha_t} ~=~ \frac{\gamma_t}{1 - \left( \frac{1 + \alpha_t}{2} \right )}, \]which shows that $\gamma_t + K\left(\frac{1 + \alpha_t}{2}\right) \leq K$. Hence, 
\begin{align*}
\expect{ \exp \left(  \lambda X_{t+1}\right ) } ~\leq~ \exp \left( \lambda K \right), 
\end{align*}
as desired.
\end{proof}
Now we are ready to complete the proof of both claims in \Theorem{RecursiveStochasticProcess}.The first claim from \Theorem{RecursiveStochasticProcess} follows by observing our MGF bound on $X_t$ and then applying the transition from MGF bounds to tail bounds given by \Claim{MGFToTailBound}.

Next, we prove the second claim from \Theorem{RecursiveStochasticProcess}. \Claim{MGFBoundToyProblem} gives that for every $t$ and for all $\lambda \leq 1/(\sigma_t K)$, we have $\expect{\exp \left ( \lambda \sigma_t X_t  \right )} \leq \exp \left  ( 
\lambda \sigma_t K \right )$. Hence, we can combine these MGF bounds using \Lemma{MGFHolder} to obtain $\expect{\exp \left ( \lambda \sum_{t =1}^T \sigma_t X_t \right )} \leq \exp \left ( \lambda K \sum_{t=1}^T \sigma_t \right )$ for all $\lambda \leq \left( K \sum_{t=1}^T \sigma_t \right)^{-1}$. With this MGF bound in hand, we may apply the transition from MGF bounds to tail bounds given by \Claim{MGFToTailBound} to complete the proof of the second claim from \Theorem{RecursiveStochasticProcess}.
\end{proofof}

%% file: appendix_ub_omitted.tex
\section{Omitted proofs from \Section{ubshort}}
\AppendixName{ub_omitted}
The following lemma is standard.
\begin{lemma}
\LemmaName{StandardSGDAnalysis}
Let $f$ be an 1-strongly convex and 1-Lipschitz function. Consider running \Algorithm{SGD} for $T$ iterations. Then, for every $w \in \cX$ and every $k \in [T]$,
\[
\sum_{t=k}^{T} \bigg [ f(x_t) - f(w) \bigg] 
	~\leq~ \frac{1}{2} \sum_{t=k}^T {\eta_t}{\norm{\hat{g}_t}^2} + \frac{1}{2 \eta_k} \norm{x_k - w}^2 + \sum_{t=k}^T \inner{\hat{z}_t}{x_t - w}.
\]
\end{lemma}
\begin{proof}
\begin{align*}
f(x_t) - f(w) 
	&~\leq~ \inner{g_t}{x_t - w} - \frac{1}{2} \norm{x_t - w}^2 \qquad\text{(by strong-convexity)} \\
    &~=~ \inner{\hat{g}_t}{x_t - w} - \frac{1}{2} \norm{x_t - w}^2  + \inner{\hat{z}_t}{x_t - w} \qquad(\hat{g}_t = g_t - \hat{z}_t) \\
    &~=~ \frac{1}{\eta_t} \inner{x_t - y_{t+1}}{x_t - w} - \frac{1}{2}\norm{x_t - w}^2 + \inner{\hat{z}_t}{x_t - w} \qquad(y_{t+1} = x_t - \eta_t \hat{g}_t) \\
    &~=~ \frac{1}{2\eta_t} \bigg ( \norm{x_t - y_{t+1}}^2 + \norm{x_t - w}^2 - \norm{y_{t+1} - w}^2  \bigg ) - \frac{1}{2}\norm{x_t - w}^2  + \inner{\hat{z}_t}{x_t - w} \\ 
    &~\leq~ \frac{1}{2\eta_t} \bigg ( \norm{\eta_t \hat{g}_t}^2 + \norm{x_t - w}^2 - \norm{x_{t+1} - w}^2  \bigg ) - \frac{1}{2}\norm{x_t - w}^2  + \inner{\hat{z}_t}{x_t - w}.
\end{align*}
Now, summing $t$ from $k$ to $T$,
\begin{align*}
&\sum_{t = k}^T \bigg [ f(x_t) - f(w) \bigg ] \\
	& \quad \leq~ \frac{1}{2}\sum_{t=k}^T {\eta_t}\norm{\hat{g}_t}^2  + \frac{1}{2}\sum_{t = k+1}^T \underbrace{\bigg ( \frac{1}{\eta_t} - \frac{1}{\eta_{t-1}} - 1 \bigg)}_{= 0} \norm{x_t - w}^2 + \bigg (\frac{1}{2\eta_k}  -\frac{1}{2}\bigg )\norm{x_k - w}^2 + \sum_{t=k}^T \inner{\hat{z}_t} {x_t - w} \\
    & \quad \leq~ \frac{1}{2}\sum_{t=k}^T {\eta_t}\norm{\hat{g}_t}^2   + \frac{1}{2\eta_k} \norm{x_k - w}^2 + \sum_{t=k}^T \inner{\hat{z}_t} {x_t - w}  \qquad (\eta_t = 1/t), \\ 
\end{align*}as desired.
\end{proof}

\begin{proofof}{\Lemma{ShamirZhangHalf}}
Let $k \in [T-1]$. Apply \Lemma{StandardSGDAnalysis}, replacing $k$ with $T-k$ and $w = x_{T-k}$ to obtain: 
$$
\sum_{t = T-k}^T \bigg [ f(x_t) - f(x_{T-k}) \bigg ] ~\leq~ \frac{1}{2}\sum_{t=T-k}^T \eta_t \norm{\hat{g}_t}^2 + \sum_{t = T-k}^T \inner{\hat{z}_t}{x_t - x_{T-k}}.
$$Now, divide this by $k+1$ and define $S_k = \frac{1}{k+1} \sum_{t = T-k}^T f(x_t)$ to obtain
$$
S_k - f(x_{T-k}) 
	~\leq~ \frac{1}{2(k+1)}\sum_{t=T-k}^T \eta_t \norm{\hat{g}_t}^2 + \frac{1}{k+1}\sum_{t = T-k}^T \inner{\hat{z}_t}{x_t - x_{T-k}}
$$Observe that $kS_{k-1} = (k+1)S_k - f(x_{T-k})$. Combining this with the previous inequality yields
$$
kS_{k-1} ~=~ kS_k + \big( S_k - f(x_{T-k} ) \big) ~\leq~ kS_k + \frac{1}{2(k+1)}\sum_{t=T-k}^T \eta_t \norm{\hat{g}_t}^2 + \frac{1}{k+1}\sum_{t = T-k}^T \inner{\hat{z}_t}{x_t - x_{T-k}}.  
$$Dividing by $k$, we obtain:
$$S_{k-1} ~\leq~ S_k + \frac{1}{2k(k+1)}\sum_{t=T-k}^T \eta_t \norm{\hat{g}_t}^2 + \frac{1}{k(k+1)}\sum_{t = T-k}^T \inner{\hat{z}_t}{x_t - x_{T-k}}.
$$ Thus, by induction:
\begin{align*}
f(x_T) 
	&~=~ S_{0} \\
    &~\leq~S_{T/2} + \sum_{k=1}^{T/2} \frac{1}{2k(k+1)}\sum_{t=T-k}^T \eta_t \norm{\hat{g}_t}^2 + \sum_{k=1}^{T/2} \frac{1}{k(k+1)}\sum_{t = T-k}^T \inner{\hat{z}_t}{x_t - x_{T-k}} \\
    &~=~  \frac{1}{ T/2   + 1    } \sum_{t = T/2}^T f(x_t) + \sum_{k=1}^{T/2} \frac{1}{2k(k+1)}\sum_{t=T-k}^T \eta_t \norm{\hat{g}_t}^2 + \sum_{k=1}^{T/2} \frac{1}{k(k+1)}\sum_{t = T-k}^T \inner{\hat{z}_t}{x_t - x_{T-k}}.
\end{align*}
Note that $\norm{\hat{g}_t}^2 \leq 4$ and $\eta_t = 1/t$.
So we can bound the middle term as
\begin{align*}
\sum_{k=1}^{T/2} \frac{1}{2k(k+1)}\sum_{t=T-k}^T \eta_t \norm{\hat{g}_t}^2
&
\leq 2 \sum_{k=1}^{T/2} \frac{1}{k(k+1)}\sum_{t=T-k}^T \frac{1}{t} \\
& \leq 2 \sum_{k=1}^{T/2} \frac{1}{k(T-k)} \\
& = \frac{2}{T} \sum_{k=1}^{T/2} \left( \frac{1}{k} + \frac{1}{T-k} \right) \\
& = O\left( \frac{\log T}{T} \right).
\end{align*}
This completes the proof.
\end{proofof}

\begin{proofof}{\Claim{ToyProblem}}
We begin by stating two consequences of strong convexity: 

\begin{enumerate}
\item $\langle g_t, x_t - x^* \rangle \geq f(x_t) - f(x^*) + \frac{1}{2}\norm{x_t - x^*}^2$,
\item $f(x_t) - f(x^*) \geq \frac{1}{2}\norm{x_t - x^*}^2$ \qquad(since $0 \in \partial f(x^*)$).
\end{enumerate}
The analysis proceeds as follows:
\begin{align*}
\norm{x_{t+1} - x^*}^2 
	&~=~ \norm{\Pi_{\cX}(x_t - \eta_t\hat{g}_t) - x^*}^2 \\
    &~\leq~ \norm{x_t - \eta_t \hat{g}_t - x^*}^2 \quad \text{(\Claim{ProjContraction})}\\
    &~=~ \norm{x_t - x^*}^2 - 2\eta_t \langle \hat{g}_t, x_t - x^* \rangle + \eta_t^2 \norm{\hat{g}_t}^2 \\
    &~=~ \norm{x_t - x^*}^2 - 2 \eta_t \langle g_t, x_t - x^* \rangle + 2 \eta_t \langle \hat{z}_t, x_t - x^* \rangle + \eta_t^2 \norm{\hat{g}_t}^2 \\
    &~\leq~ \norm{x_t - x^*}^2 -2\eta_t \biggr ( f(x_t) - f(x^*) \biggr ) - \frac{1}{t} \norm{x_t - x^*}^2 + 2 \eta_t \langle \hat{z}_t, x_t - x^* \rangle + \eta_t^2 \norm{\hat{g}_t}^2 \\
    &~\leq~ \biggr ( 1 - \frac{2}{t} \biggr )\norm{x_t - x^*}^2 + 2\eta_t \langle \hat{z}_t, x_t - x^* \rangle + \eta_t^2 \norm{\hat{g}_t}^2 \\
    &~=~ \biggr ( \frac{t-2}{t} \biggr ) \frac{Y_{t-1}}{t-1} + \frac{2}{ t} U_{t-1} \sqrt{\frac{Y_{t-1}}{t-1}} +  \frac{\norm{\hat{g}_t}^2}{ t^2} .
\end{align*}
Recall that $\norm{\hat{g}_t}^2 \leq 4$ because $\hat{z}_t \leq 1$ and $f$ is 1-Lipschitz. Multiplying through by $t$ and bounding $\norm{\hat{g}_t}^2$ by 4 yields the desired result. 
\end{proofof}

\subsection{Proof of \Lemma{w_t_upper_bound}}
\AppendixName{w_t_upper_bound}
\begin{proofof}{\Lemma{w_t_upper_bound}}
Recall from \Section{ubshort} that $\alpha_j = \frac{1}{(T-j)(T-j+1)}$ and $w_t = \sum_{j=T/2}^{t-1} \alpha_j (x_t - x_j)$.
\begin{definition}
\DefinitionName{B_TDefinition}
Define $B_T \coloneqq \sum_{t=T/2}^T \frac{1}{T-t+1} \sum_{j=T/2}^{t-1} \alpha_j \norm{x_t - x_j}^2$.
\end{definition}
\begin{claim}
\ClaimName{B_T}
$\sum_{t=T/2}^T \norm{w_t}^2 \leq B_T$.
\end{claim}
\begin{proof}
Let $A_t = \sum_{j=T/2}^{t-1} \alpha_j$.
Then
\begin{align*}
\norm{w_t}^2
& = A^2 \norm{\sum_{j=T/2}^{T-1} \frac{\alpha_j}{A} (x_t - x_j)}^2 \\
& \leq A^2 \sum_{j=T/2}^{T-1} \frac{\alpha_j}{A} \norm{x_t - x_j}^2 \\
& \leq \frac{1}{T-t+1} \sum_{j=T/2}^{t-1} \alpha_j \norm{x_t - x_j}^2,
\end{align*}
where the first inequality is due to the convexity of $\norm{\cdot}^2$ and the second inequality is \Claim{SumOfAlphas}.
\end{proof}

\repeatclaim{\Lemma{StochasticDistanceEstimate}}{\distanceestimate}

\begin{proofof}{\Lemma{StochasticDistanceEstimate}}
\begin{align*}
\norm{x_a - x_b }^2
    	&~=~ \norm{   x_a - \Pi_{\cX}(y_b) }_2^2 \\
        &~\leq~ \norm{  x_a - y_b }_2^2 \quad \text{(\Claim{ProjContraction})} \\
        &~=~ \norm{x_a - x_{b-1} + x_{b-1} - y_b  }_2^2 \\
        &~=~ \norm{x_a - x_{b-1}}_2^2 + \norm{x_{b-1} - y_b}_2^2 + 2\inner{\eta_{b-1} \hat{g}_{b-1} }{x_a - x_{b-1}} \\
		&~=~ \norm{x_a - x_{b-1}}_2^2 + \eta_{b-1}^2\norm{\hat{g}_{b-1}}_2^2 + 2\inner{\eta_{b-1} \hat{g}_{b-1} }{x_a - x_{b-1}} \\ 
        &~=~ \norm{x_a - x_{b-1}}_2^2 + \eta_{b-1}^2\norm{\hat{g}_{b-1}}_2^2 + 2\inner{\eta_{b-1} g_{b-1} }{x_a - x_{b-1}}  + 2\inner{\eta_{b-1}\hat{z}_{b-1}}{x_{b-1} - x_a}
\end{align*}
Repeating this argument iteratively on $\norm{x_a - x_{b-1}}$, $\norm{x_a - x_{b-2}}$, \ldots , $\norm{x_a - x_{a+1}}$, we obtain:
$$
\norm{x_a - x_b }^2 \leq \sum_{i=a}^{b-1} \frac{\norm{\hat{g}_i}_2^2}{i^2} + 2 \sum_{i=a}^{b-1} \frac{\inner{g_i}{x_a - x_i}}{{i}} + 2 \sum_{i=a}^{b-1} \frac{\inner{\hat{z}_i}{x_i - x_a}}{{i}}. 
$$
Applying the inequality $\inner{g_i}{ x_a - x_i} \leq f(x_a) - f(x_i)$ to each term of the second summation gives the desired result.
\end{proofof}

Using \Lemma{StochasticDistanceEstimate} and the bound $\norm{\hat{g}_t}^2 \leq 4$ for all $t$, let us write $B_T \leq \Lambda_1 + \Lambda_2 + \Lambda_3$ where
\begin{align*}
\Lambda_1 ~&:=~ 4 \sum_{t =T /2  }^T \frac{1}{T-t+1} \sum_{j =  T/2 }^{t-1} \alpha_j \sum_{i = j}^{t-1} \frac{1}{i^2} ,\\
\Lambda_2 ~&:=~ 2 \sum_{t = T/2}^T \frac{1}{T-t+1} \sum_{j = T/2}^{t-1} \alpha_j \sum_{i = j}^{t-1} \frac{\big ( F_j - F_i \big ) }{i} \qquad\text{(where $F_a := f(x_a) - f(x^*)$  )}, \\
\Lambda_3 ~&:=~ 2 \sum_{t = T /2}^T \frac{1}{T-t+1} \sum_{j = T/2}^{t-1} \alpha_j \sum_{i = j}^{t-1} \frac{\inner{\hat{z}_i}{x_i - x_j}}{i}.
\end{align*}

Let us bound each of the terms separately.
\begin{claim}
\ClaimName{Lambda1Bound}
$\Lambda_1 \leq  O \bigg ( \frac{\log^2(T) }{T^2} \bigg ) $.
\end{claim}
\begin{proof}
This follows from some straightforward calculations.
Indeed,
\begin{align*}
\Lambda_1
	~&=~ 4 \sum_{t = T/2 }^T \frac{1}{T-t+1} \sum_{j = T/2  }^{t-1} \alpha_j \sum_{i = j}^{t-1} \frac{1}{i^2} \\ 
    ~&\leq~ 4 \sum_{t = T/2  }^T \frac{1}{T-t+1} \sum_{j = T/2 }^{t-1} \frac{1}{(T-j)(T-j+1)} \frac{(T-j)}{ (T/2)^2} \\ 
    ~&\leq~ \frac{4}{( T/2)^2} \sum_{t = T/2  }^T \frac{1}{T-t+1}\sum_{j = T/2 }^{t-1} \frac{1}{T-j+1} \\
    ~&\leq~ O \bigg ( \frac{\log^2(T)}{T^2} \bigg ). \qedhere
\end{align*}
\end{proof}

\begin{claim}
\ClaimName{Lambda2Bound}
$$
\Lambda_2 ~\leq~ O \bigg( \frac{\log(T) }{T^2} \bigg) +O \bigg ( \frac{\log(T) }{T}\bigg ) \norm{x_{T/2} - x^* }_2^2 + O\bigg (\frac{\log(T) }{T^2} \bigg ) \sum_{t = T/2}^{T-1} \inner{\hat{z}_t}{x_t - x^*}. 
$$
\end{claim}
We will prove \Claim{Lambda2Bound} in the next section.

\begin{claim}
\ClaimName{Lambda3}
$$ 
\Lambda_3 ~=~ \sum_{i = T/2  }^{T-1} \inner{\hat{z}_{i}}{\frac{C_i}{i} w_i},
$$ where $C_i \coloneqq  \sum_{\ell= i +1}^T \frac{2}{T-i+1} =  O \big(  \log(T) \big)$.
\end{claim}
\begin{proof}
Rearranging the order of summation in $\Lambda_3$ we get: 
\begin{align*}
\Lambda_3 
	~&=~ \sum_{t = T/2 }^T \frac{2}{T - t + 1} \sum_{j=  T/2 }^{t-1} \alpha_j \sum_{i = j}^{t-1} \frac{\inner{\hat{z}_i}{x_i - x_j}}{i} \\
    ~&=~ \sum_{t= T/2 }^T \frac{2}{T-t+1} \sum_{i= T/2  }^{t-1} \frac{\inner{\hat{z}_i}{ \sum_{j =T/2  }^{i-1} \alpha_j(x_i - x_j) }}{i} \\
    ~&=~ \sum_{t=  T/2}^T \frac{2}{T-t+1} \sum_{i= T/2}^{t-1} \frac{\inner{\hat{z}_i}{ w_i }}{i} \\
    ~&=~ \sum_{i = T/2   }^{T-1} \inner{  \hat{z}_i  }{   \frac{\bigg ( \sum_{t= i+1}^T \frac{2}{T-t+1} \bigg) }{i}  w_i  } \\
    ~&=~ \sum_{i = T/2 }^{T-1} \inner{  \hat{z}_i  }{    \frac{C_i}{i} w_i  },
\end{align*}
as desired.
\end{proof}
The previous three claims and the fact that $B_T$ is an upper bound on $\sum_{t=T/2}^T \norm{w_t}^2$ (\Claim{B_T}) complete the proof of \Lemma{w_t_upper_bound}.
\end{proofof}

\subsection{Proof of \Claim{Lambda2Bound}}

Let us rewrite 
$$
\Lambda_2 ~=~ \sum_{a  = T/2 }^{T-1} \gamma_a F_a
$$ and determine the coefficients $\gamma_a$.

\begin{claim}
\ClaimName{GammaBound}
For each $a \in \{ \floor { T/2 }  , \ldots , T-1 \}$, $\gamma_a ~=~ O \bigg (\frac{\log(T) }{T^2} \bigg ).$
\end{claim}

\begin{proof}
In the definition of $\Lambda_2$, the indices providing a positive coefficient for $F_a$ must satisfy $j = a$, $i \leq a$, and $a \leq t-1$. Hence, the positive contribution to $\gamma_a$ is: 

\begin{align*}
&\sum_{t=1+a}^T \frac{2}{T-t+1} \alpha_a \sum_{i=a}^{t-1} \frac{1}{i}  \\ 
&~\leq~ \sum_{t=1+a}^T \bigg (\frac{2}{T-t+1} \alpha_a \bigg) \bigg (\log \big(  T/(a - 1) \big) \bigg ) \qquad\text{(by \Claim{HalfHarmonicSum})} \\
&~\leq~  \sum_{t = 1  + a}^T \bigg(  \frac{2}{T-t+1} \alpha_a \bigg)\bigg( \frac{T - a + 1}{a-1} \bigg ) \qquad\text{(by \Claim{logToFraction})}\\
&~=~ \sum_{t = 1 + a}^T \bigg (\frac{2}{T-t+1} \bigg ) \bigg ( \frac{1}{(T-a)(T-a+1)} \bigg ) \bigg (\frac{T-a+1}{a-1} \bigg ) \\
&~=~ \frac{1}{T-a} \sum_{t = 1 +a }^T \frac{2}{(T-t+1)(a-1)}
\end{align*}
The terms contributing to the negative portion of $\gamma_a$ satisfy, $i =a$, $j \leq a$, and $a \leq t-1$. The negative contribution can be written as
\begin{align*}
&-\sum_{t = 1 + a}^T \frac{2}{T-t+ 1}\sum_{j= T/2 }^a \alpha_j \frac{1}{a } \\ 
&~=~ - \sum_{t = 1+a}^T \bigg (\frac{2}{T-t + 1} \bigg ) \bigg (\frac{1}{a } \bigg ) \bigg (\frac{1}{T-a} - \frac{1}{T/2  + 1} \bigg) \\
&~=~ - \sum_{t = 1 + a}^T \bigg ( \frac{2}{T-t + 1}  \bigg ) \bigg (\frac{1}{a} \bigg ) \bigg (\frac{2a - T   +2}{2(T/2  + 1)(T-a)} \bigg ) \\ 
&~=~ - \frac{1}{(T/2+ 1)(T-a)}\sum_{t = 1 + a}^T \bigg (\frac{2}{T-t+1} \bigg ) \bigg ( \frac{2a -T + 2}{2a} \bigg) \\
&~=~ - \frac{2}{(T  + 2)(T-a)}\sum_{t = 1 + a}^T \bigg (\frac{2}{T-t+1} \bigg ) \bigg ( 1 - \frac{T - 2}{2a} \bigg) 
\end{align*}
where on the last line we used $T-1 \leq  2\floor{T/2}  \leq T$. Now, combining the positive and negative contribution we see:
\begin{align*}
\gamma_a 
	&~\leq~ \frac{1}{T-a} \sum_{t = 1+a}^T \frac{2}{T-t+1} \biggr ( \frac{1}{a-1} - \frac{2}{T+2} \bigg( 1 - \frac{T-2}{2a}\bigg) \biggr)  \\
	&~=~ \frac{1}{T-a} \sum_{t = 1 + a}^T \frac{2}{T-t+1} \biggr ( \frac{T+2 - 2\big(a -1 \big)\big( 1 - \frac{T-2}{2a}\big)}{(a-1)(T+2)} \biggr) \\
    &~=~ \frac{1}{T-a}\sum_{t= 1 +a}^T \frac{2}{T-t+1} \biggr( \frac{T + 2 - 2(a-1) + \frac{2(T-2)(a-1)}{2a}}{(a-1)(T+2)} \biggr ) \\
    &~\leq~ \frac{1}{T-a}\sum_{t = 1+a}^T \frac{2}{T-t+1} \biggr ( \frac{2 \big(T -a\big) + 2}{(T+2)(a-1)} \biggr ) \\
    &~\leq~ \frac{1}{T-a}\sum_{t = 1+a}^T \frac{2}{T-t+1}\biggr ( \frac{2\big(T-a\big) + 2 \big( T-a \big) }{(T+2)(a-1)} \biggr )  \qquad (a \leq T-1) \\
    &~=~ \frac{1}{(T+2)(a-1)} \sum_{t=1+a}^T \frac{4}{T-t+1}  \\
    &~\leq~ \frac{2}{(T+2)(T-2)} \sum_{t=1+a}^T \frac{4}{T-t+1} \qquad (a \geq T/2) \\
    &~=~  O \bigg (\frac{\log(T)}{T^2} \bigg ),
\end{align*}as desired.
\end{proof}

\begin{proofof}{\Claim{Lambda2Bound}}
\begin{align*}
\Lambda_2 
	&~=~ \sum_{a = T/2 }^{T-1} \gamma_a F_a  \\ 
    &~\leq~ O\bigg( \frac{ \log(T) }{T^2} \bigg) \sum_{a =  T/2 }^{T-1} f(x_a) - f(x^*) \qquad\text{(by \Claim{GammaBound})} \\
    &~\leq~   O \bigg( \frac{ \log(T) }{T^2} \bigg) \bigg (   \frac{1}{2}   \sum_{t = T/2 }^{T-1} \eta_{t} \norm{\hat{g}_t}_2^2 + \frac{1}{2 \eta_{ T/2 } } \norm{ x_{T/2 } - x^* }_2^2 + \sum_{t = T/2 }^{T-1} \inner{\hat{z}_t}{x_t - x^*}   \bigg ) \qquad\text{(by \Lemma{StandardSGDAnalysis})} \\
    &~\leq~ O\bigg( \frac{ \log(T) }{T^2} \bigg) \sum_{t =  T/2 }^{T-1} \frac{1}{t} + O\bigg( \frac{ \log(T) }{T} \bigg) \norm{x_{  T/2 } - x^* }_2^2 + O\bigg( \frac{ \log(T) }{T^2} \bigg) \sum_{t = T/2 }^{T-1} \inner{\hat{z}_t}{x_t - x^*} \quad\text{($\norm{\hat{g}_t}_2 \leq 2$)} \\
    &~\leq~ O\bigg( \frac{ \log(T) }{T^2} \bigg) + O\bigg( \frac{ \log(T) }{T} \bigg) \norm{x_{T/2 } - x^* }_2^2 + O\bigg( \frac{ \log(T) }{T^2} \bigg) \sum_{t = T/2 }^{T-1} \inner{\hat{z}_t}{x_t - x^*},
\end{align*}as desired.
\end{proofof}

%% file: generalizations.tex
\section{Generalizations}
\AppendixName{Generalizations}

In this section, we discuss generalizations of our results.
In \Subsection{Reduction}, we explain that the scaling of the function (e.g., Lipschitzness) can be normalized without loss of generality.
In \Subsection{SubGaussian}, we explain how the assumption of almost surely bounded noise can be relaxed to sub-Gaussian noise in our upper bounds
(Theorems~\ref{thm:FinalIterateHighProbability}, \ref{thm:FinalIterateHighProbabilityLipschitz} and \ref{thm:SuffixAverageHighProbability}).

\subsection{Scaling assumptions}
\SubsectionName{Reduction}

For most of this paper we consider only convex functions that have been appropriately normalized, due to the following facts.
\begin{itemize}
\item \textbf{Strongly convex case.}
The case of an $\alpha$-strongly convex and $L$-Lipschitz function can be reduced to the case of a $1$-strongly convex and $1$-Lipschitz function.

\item \textbf{Lipschitz case.}
The case of an $L$-Lipschitz function on a domain of diameter $R$ can be reduced to the case of a $1$-Lipschitz function on a domain of diameter $1$.
\end{itemize}
We will discuss only the first of these in detail. The second is proven with similar ideas.

The main results from this section are as follows.

\begin{theorem}
\TheoremName{LastIterateHighProbabilityGeneral}
Suppose $f$ is $\alpha$-strongly convex and $L$-Lipschitz, and that $\hat{z}_t$ has norm at most $L$ almost surely. Consider running \Algorithm{SGD} for $T$ iterations with step size $\eta_t = \frac{1}{\alpha t}$. Let $x^* = \argmin_{x \in \cX} f(x)$. Then, with probability at least $1- \delta$,
$$
f(x_{T+1}) - f(x^*) ~\leq~ O \bigg ( \frac{L^2}{\alpha} \frac{\log(T)\log(1/\delta)}{T} \bigg ).
$$
\end{theorem}

\begin{theorem}
\TheoremName{SuffixAverageHighProbabilityGeneral}
Suppose $f$ is $\alpha$-strongly convex and $L$-Lipschitz, and that $\hat{z}_t$ has norm at most $L$ almost surely. Consider running \Algorithm{SGD} for $T$ iterations with step size $\eta_t = \frac{1}{\alpha t}$. Let $x^* = \argmin_{x \in \cX} f(x)$. Then, with probability at least $1- \delta$, 
$$
f\bigg(\frac{1}{T/2+1} \sum_{t = T/2}^{T} x_t \bigg ) - f(x^*) ~\leq~ O \bigg (\frac{L^2}{\alpha} \frac{\log(1/\delta)}{T}  \bigg). 
$$
\end{theorem}

We prove these theorems by reduction to \Theorem{FinalIterateHighProbability} and \Theorem{SuffixAverageHighProbability}, respectively. That is, suppose that $f$ is a function that has strong convexity parameter $\alpha$ and Lipschitz parameter $L$. We construct a function $g$ that is 1-Lipschitz and 1-strongly convex (using \Claim{1Lipschitz1StrongConvexFromF}) and a subgradient oracle such that running SGD on $g$ with this subgradient oracle is equivalent to running SGD on $f$. Formally, we show the following:

\begin{claim}
\ClaimName{Reduction}
Suppose $f$ is $\alpha$-strongly convex and $L$-Lipschitz on a domain $\cX \subset \bR^n$.
Let the initial point $x_1 \in \cX$ be given.
Let $g$ be as defined in  \Claim{1Lipschitz1StrongConvexFromF}. Then, there is a coupling between the following two processes:
\begin{itemize}
\item the execution of \Algorithm{SGD} on input $f$ with initial point $x_1$, step size $\eta_t = 1/(\alpha t)$ and convex set $\cX$
\item the execution of \Algorithm{SGD} on input $g$ with initial point $\tilde{x}_1 := (\alpha/L)x_1$, step size $\tilde{\eta}_t = 1/t$ and convex set $(\alpha/L) \cX$
\end{itemize}
such that the iterates of the second process correspond to the iterates of the first process scaled by $\alpha/L$.  That is, if we denote by $\tilde{x}_t$ the iterates of the execution of SGD using $g$ and $x_t$ for the execution on $f$, then $\tilde{x}_t = (\alpha / L) x_t$. 
\end{claim}

Now, suppose we are given an $\alpha$-strongly convex and $L$-Lipschitz function, $f$, an initial point $x_1$ and a convex set $\cX$. We obtain \Theorem{LastIterateHighProbabilityGeneral} and \Theorem{SuffixAverageHighProbabilityGeneral} by performing the above coupling and executing SGD on the 1-Lipschitz and 1-strongly convex function. We may apply our high probability upper bounds to this execution of SGD because it satisfies the assumptions of \Theorem{FinalIterateHighProbability} and \Theorem{SuffixAverageHighProbability}. Finally, because of \Claim{Reduction}, we can reinterpret the iterates of the execution of SGD on $g$ as a scaled version of the iterates of the execution of SGD on $f$. This immediately proves \Theorem{LastIterateHighProbabilityGeneral} and \Theorem{SuffixAverageHighProbabilityGeneral}. Now, let us prove \Claim{Reduction}. 

\begin{proofof}{\Claim{Reduction}}
The coupling is given by constraining the algorithms to run in parallel and enforcing the execution of SGD on $g$ to use a scaled version of the outputs of the subgradient oracle used by the execution of SGD on $f$. That is, at step $t$, if $\hat{g}_t$ is the output of the subgradient oracle of the execution of SGD on $f$, then we set the output of the subgradient oracle of the execution of SGD on $g$ at step $t$ to be $\frac{1}{L} \hat{g}_t$.

In order for this coupling to make sense, we have to ensure that this subgradient oracle for $g$ is valid. That is, we must show that at each step, the subgradient oracle we define for $g$ returns a true subgradient in expectation, and that the noise of this subgradient oracle is at most 1 with probability 1. We show by induction, that at each step $\tilde{x}_t = (\alpha/L) x_t$.  

By definition, $\tilde{x}_1 = (\alpha/L)x_1$. Now, assume $\tilde{x}_t  = (\alpha/L)x_t$. Let $\hat{g}_t$ be the output of the subgradient oracle for SGD running on $f$. The subdifferential for $g$ at $\tilde{x}_t$ is $\frac{1}{L} \partial f(x_t)$ using the chain rule for subdifferentials. Therefore, the subgradient oracle for $g$ is certainly valid at this step. Now, $y_{t+1} = x_{t} - \frac{1}{\alpha t} \hat{g}_t$. Meanwhile, $\tilde{y}_{t+1} = \tilde{x}_t - \frac{1}{t} \frac{1}{L} \hat{g}_t = \frac{\alpha}{L} (  x_t - \frac{1}{\alpha t} \hat{g}_t) = \frac{\alpha}{L} y_{t+1}$. Therefore, 
$$
\tilde{x}_{t+1} = \Pi_{(\alpha/L)\cX}(\tilde{y}_{t+1})  = \Pi_{(\alpha/L)\cX}( y_{t+1} (\alpha/L) ) = (\alpha/L) \Pi_{\cX}(y_{t+1}) = (\alpha/L) x_{t+1}
$$as desired.
\end{proofof}

\begin{claim}
\ClaimName{1Lipschitz1StrongConvexFromF}
Let $f$ be an $\alpha$-strongly convex and $L$-Lipschitz function. Then, $g(x) \coloneqq \frac{\alpha}{L^2}f(\frac{L}{\alpha}x )$ is $1$-Lipschitz and $1$-strongly convex.
\end{claim}

\begin{proof}
First we show that $g$ is $1$-Lipschitz:
$$
\Abs{g(x) - g(y)} ~=~ \frac{\alpha}{L^2} \Abs{f\bigg(\frac{L}{\alpha}x \bigg) - f\bigg(\frac{L}{\alpha}y \bigg)   } ~\leq~ \frac{\alpha}{L^2} L \norm{\frac{L}{\alpha} (x - y) } ~=~ \norm{x-y}.
$$
The inequality holds since $f$ is $L$-Lipschitz.

Now we show that $g$ is $1$-strongly convex. A function $h$ is $\alpha$ strongly convex, if and only if the function $x \mapsto h(x) - \frac{\alpha}{2}\norm{x}^2$ is convex. Indeed, for $g$:
$$
g(x) - \frac{1}{2}\norm{x}^2 ~=~ \frac{\alpha}{L^2}f \bigg(\frac{L}{\alpha}x \bigg) - \frac{1}{2}\norm{x}^2 ~=~ \frac{\alpha}{L^2} \bigg(   f \bigg(\frac{L}{\alpha} x \bigg) - \frac{L^2}{2\alpha} \norm{x}^2  \bigg ) ~=~ \frac{\alpha}{L^2} \bigg(   f \bigg(\frac{L}{\alpha} x \bigg) - \frac{\alpha}{2} \norm{\frac{L}{\alpha}x}^2  \bigg ).
$$The function on the right is convex because $f$ is $\alpha$-strongly convex. This implies that $x \mapsto g(x) - \frac{1}{2}\norm{x}^2$ is convex, meaning that $g$ is 1-strongly convex.
\end{proof}


\subsection{Sub-Gaussian Noise}
\SubsectionName{SubGaussian}

In this section, we relax the assumption that $\norm{\hat{z}_t} \leq 1$ with probability 1 and instead assume that for each $t$, $\hat{z}_t$ is sub-Gaussian conditioned on $\cF_{t-1}$.  The proof of the extensions are quite easy, given the current analyses. See the full version of our paper for statements and proofs of this extension.

\paragraph{Main ideas.} Most of our analyses can remain unchanged. The main task at hand is identifying the places where we use the upper bound $\norm{\hat{z}_t} \leq 1$ outside of the MGF analyses (using this bound inside an MGF is morally the same using the fact that $\norm{\hat{z}_t}$ is sub-Gaussian). The main culprit is that we often bound $\norm{\hat{g}_t}^2$ by 4. Instead we must carry these terms forward and handle them using MGFs. The consequences of this are two-fold. Firstly, this introduces new MGFs to bound, but intuitively these are easy to bound because the terms they were involved in in the original analysis were sufficiently bounded and therefore their MGFs should now also be sufficiently bounded. Furthermore, removing these constant bounds results in many of our MGF expressions to include more random terms which we previously ignored and pulled out of our MGF arguments because they were constant. But again, these terms can be dealt with by first isolating them by applying an MGF triangle inequality (using H\"older or Cauchy-Schwarz) and then bounding their MGF. 

%% file: appendix_lb_delta.tex
\section{Necessity of \texorpdfstring{$\log(1/\delta)$}{log(1/delta)}}
\AppendixName{lb_delta}
In this section, we show that the error of the last iterate and suffix average of SGD is $\Omega(\log(1/\delta) / T)$ with probability at least $\delta$.
\begin{lemma}[{\cite[Lemma~4]{RevChernoff}}]
\LemmaName{ReverseChernoff}
Let $X_1, \ldots, X_T$ be independent random variables taking value $\{-1, +1\}$ uniformly at random and $X = \frac{1}{T} \sum_{t=1}^T X_i$.
Then for any $0 < c < O(\sqrt{T})$,
\[
\prob{X \geq \frac{c}{\sqrt{T}}} \geq \exp(-9c^2/2).
\]
\end{lemma}
Consider the single-variable function $f(x) = \frac{1}{2}x^2$ and suppose that the domain is $\cX = [-1, 1]$.
Then $f$ is 1-strongly convex and 1-Lipschitz on $\cX$.
Moreover, suppose that the subgradient oracle returns $x - \hat{z}$ where $\hat{z}$ is $-1$ or $+1$ with probability $1/2$ (independently from all previous calls to the oracle).
Finally, suppose we run \Algorithm{SGD} with step sizes $\eta_t = 1/t$ with an initial point $x_1 = 0$.
\begin{claim}
If $T \geq O(\log(1/\delta))$ then $f(x_{T+1}) \geq \Omega(\log(1/\delta) / T)$ with probability at least $\delta$.
\end{claim}
\begin{proof}
We claim that $x_{t+1} = \frac{1}{t} \sum_{i=1}^t \hat{z}_i$ for all $t \in [T]$ where $\hat{z}_i$ is the random sign returned by the subgradient oracle at iteration $i$.
Indeed, for $t = 1$, we have $y_2 = x_1 - \eta_1 (x_1 - \hat{z}_1) = \hat{z}_1$ since $\eta_1 = 1$.
Moreover, $x_2 = \Pi_{\cX}(y_2) = y_2$ since $\abs{y_2} \leq 1$.
Now, suppose that $x_{t} = \frac{1}{t-1} \sum_{i=1}^{t-1} \hat{z}_i$.
Then $y_{t+1} = x_t - \eta_t (x_t - \hat{z}_t) = \frac{1}{t} \sum_{i=1}^{t} \hat{z}_i$.
Since $\abs{y_{t+1}} \leq 1$, we have $x_{t+1} = y_{t+1}$.

Hence, by \Lemma{ReverseChernoff} with $c = \sqrt{\log(1/\delta)}$, we have $x_{T+1} \geq \sqrt{\log(1/\delta)} / \sqrt{T}$ with probability at least $\Omega(\delta)$ (provided $T \geq O(\log(1/\delta))$).
We conclude that $f(x_{T+1}) \geq \frac{\log(1/\delta)}{2T}$ with probability at least $\Omega(\delta)$.
\end{proof}

We can also show that \Theorem{SuffixAverageHighProbability} is tight.
To make the calculations simpler, first assume $T$ is a multiple of $4$.
We further assume that the noise introduced by the stochastic subgradient oracle is generated as follows.
For $1 \leq t < T/2$ and $t > 3T/4$, $\hat{z}_t = 0$.
For $T/2 \leq t \leq 3T/4$, first define $A_t = \sum_{i=t}^{T} \frac{1}{i}$.
Then we set $\hat{z}_t$ to be $\pm \frac{1}{4A_t}$ with probability $1/2$.
Note that $A_t \geq 1/4$ for $T/2 \leq t \leq 3T/4$ so we still have $\abs{\hat{z}_t} \leq 1$ for all $t$.
\begin{claim}
If $T \geq O(\log(1/\delta))$ then $f\left( \frac{1}{T/2+1} \sum_{t=T/2+1}^{T+1} x_t \right) \geq \Omega\left( \frac{\log(1/\delta)}{T} \right)$ with probability at least $\delta$.
\end{claim}
\begin{proof}
Proceeding as in the above claim, we have $x_{t+1} = \frac{1}{t} \sum_{i=1}^t \hat{z}_i$.
We claim that
\begin{equation}
\EquationName{lb_delta_1}
\frac{1}{T/2 + 1} \sum_{t=T/2+1}^{T+1} x_t = \frac{1}{T/2 + 1} \sum_{t=T/2}^{3T/4} A_t \hat{z}_t.
\end{equation}
To see this, we have
\begin{align*}
\frac{1}{T/2 + 1} \sum_{t=T/2}^T x_{t+1}
& = \frac{1}{T/2 + 1} \sum_{t=T/2}^T \frac{1}{t} \sum_{i=1}^T \hat{z}_i \\
& = \frac{1}{T/2 + 1} \sum_{i=1}^T \hat{z}_i \sum_{t=\max\{i,T/2\}}^T \frac{1}{t} \\
& = \frac{1}{T/2+1} \sum_{t=T/2}^{3T/4} A_t \hat{z}_t,
\end{align*}
where the last equality uses the assumption that $\hat{z}_t \neq 0$ only if $T/2 \leq t \leq 3T/4$ and changes the name of the index.
Notice that $A_t \hat{z}_t$ is $\pm \frac{1}{4}$ with probability $1/2$ so we can write \Equation{lb_delta_1} as
\[
\frac{1}{4(T/2+1)} \sum_{t=1}^{T/4+1} X_t
\]
where $X_t$ are random signs.
Applying \Lemma{ReverseChernoff} with $c = \sqrt{\log(1/\delta)}$, we conclude that \Equation{lb_delta_1} is at least $\Omega(\sqrt{\log(1/\delta)} / \sqrt{T})$ with probability at least $\Omega(\delta)$ (provided $T \geq O(\log(1/\delta))$).
So we conclude that $f\left( \frac{1}{T/2+1} \sum_{t=T/2+1}^{T+1} x_t \right) \geq \Omega\left( \frac{\log(1/\delta)}{T} \right)$ with probability at least $\Omega(\delta)$.
\end{proof}